\documentclass[3p]{elsarticle}
\usepackage{hyperref,enumerate}

\usepackage[utf8]{inputenc} 
\usepackage[T1]{fontenc}    
\usepackage{hyperref}       
\usepackage{url}            
\usepackage{booktabs}       
\usepackage{amsfonts}       
\usepackage{nicefrac}       
\usepackage{microtype}      
\usepackage{graphicx}
\usepackage{subfigure}
\usepackage{tikz}
\usetikzlibrary{arrows,shapes,positioning,shadows,trees}
\tikzset{
	basic/.style  = {draw, text width=2.5cm,  rectangle},
	root/.style   = {basic, rounded corners=2pt, thin, align=center},
	level 2/.style = {basic, rounded corners=6pt, thin,align=center, text width=7.5em},
	level 3/.style = {basic, thin, align=center, text width=7.5em}
}
\usepackage{tikz-qtree}
\usepackage[linguistics]{forest}
\tikzstyle{every node}=[draw=black]
\tikzstyle{selected}=[draw=red,fill=red!30]
\tikzstyle{optional}=[dashed,fill=gray!50]

\usepackage{amsfonts}
\usepackage{bm}
\usepackage{amsmath,amssymb,rotating, mathrsfs}
\usepackage{mathtools}
\usepackage[inline]{enumitem}

\newtheorem{theorem}{Theorem}
\newtheorem{lemma}{Lemma}
\newtheorem{prop}{Proposition}
\newtheorem{corollary}{Corollary}

\newtheorem{definition}{Definition}

\newtheorem{remark}{Remark}

\newproof{proof}{Proof}

\newcommand{\cA}{\mathcal{A}}

\newcommand{\ocAI}{\overline{\mathcal{A}_{\rI}}}
\newcommand{\ocKI}{\overline{\mathcal{K}_{\rI}}}

\newcommand{\cB}{\mathcal{B}}

\newcommand{\cE}{\mathscr{E}}
\newcommand{\sG}{\mathscr{G}}
\newcommand{\sY}{\mathscr{Y}}
\newcommand{\sU}{\mathscr{U}}
\newcommand{\sV}{\mathscr{V}}

\newcommand{\cF}{\mathcal{F}}
\newcommand{\cG}{\mathcal{G}}
\newcommand{\cH}{\mathcal{H}}
\newcommand{\cI}{\mathcal{I}}

\newcommand{\cK}{\mathcal{K}}

\newcommand{\cQ}{\mathcal{Q}}
\newcommand{\cR}{\mathcal{R}}
\newcommand{\cS}{\mathcal{S}}
\newcommand{\cT}{\mathcal{T}}
\newcommand{\cU}{\mathcal{U}}
\newcommand{\cV}{\mathcal{V}}
\newcommand{\cW}{\mathcal{W}}
\newcommand{\cX}{\mathcal{X}}
\newcommand{\cY}{\mathcal{Y}}
\newcommand{\cZ}{\mathcal{Z}}

\newcommand{\bbA}{\mathbb{A}}

\newcommand{\bbC}{\mathbb{C}}

\newcommand{\bbH}{\mathbb{H}}

\newcommand{\bbR}{\mathbb{R}}

\newcommand{\frakR}{\mathfrak{R}}

\newcommand{\fraka}{\mathfrak{a}}
\newcommand{\frakb}{\mathfrak{b}}

\newcommand{\otau}{\overline{\tau}}
\newcommand{\orho}{\overline{\rho}}

\newcommand{\oy}{\overline{y}}

\newcommand{\rI}{\mathscr{I}}

\newcommand{\rP}{\mathscr{P}}

\newcommand{\rU}{\mathscr{U}}

\DeclareMathOperator{\id}{id}

\usepackage{xparse} %
\NewDocumentCommand{\norm}{mG{2}}{\|#1\|_{#2}}

\DeclareMathOperator{\im}{im}
\DeclareMathOperator{\codim}{codim}

\DeclareMathOperator{\Span}{Span}
\DeclareMathOperator{\rank}{rank}
\DeclareMathOperator{\Gr}{Gr}

\DeclareMathOperator{\F}{F}
\DeclareMathOperator{\hsp}{hsp}
\DeclareMathOperator{\cl}{cl}

\DeclarePairedDelimiter\floor{\lfloor}{\rfloor}

\newcommand{\argmin}{\mathop{\rm argmin}}
\newcommand{\argmax}{\mathop{\rm argmax}}

\newcommand{\hv}{\hat{v}}
\newcommand{\htau}{\hat{\tau}}
\newcommand{\hT}{\hat{T}}

\newcommand{\hx}{\hat{x}}
\newcommand{\hS}{\hat{S}}

\newcommand{\ra}[1]{\renewcommand{\arraystretch}{#1}}

\newcommand\inner[2]{\langle #1, #2 \rangle}

\begin{document}
	
	\begin{frontmatter}
		
		\title{Homomorphic Sensing of Subspace Arrangements}
		
		\author{Liangzu Peng}
		\ead{penglz@shanghaitech.edu.cn}
		
		\author{Manolis C. Tsakiris}
		\ead{mtsakiris@shanghaitech.edu.cn}
		\address{SIST, ShanghaiTech University, No. 393 Huaxia Middle Road, Pudong Area, Shanghai, China}

		\begin{abstract}
			Homomorphic sensing is a recent algebraic-geometric framework that studies the unique recovery of points in a linear subspace from their images under a given collection of linear maps. It has been successful in interpreting such a recovery in the case of permutations composed by coordinate projections, an important instance in applications known as unlabeled sensing, which models data that are out of order and have missing values. In this paper, we provide tighter and simpler conditions that guarantee the unique recovery for the single-subspace case, extend the result to the case of a subspace arrangement, and show that the unique recovery in a single subspace is locally stable under noise. We specialize our results to several examples of homomorphic sensing such as real phase retrieval and unlabeled sensing. In so doing, in a unified way, we obtain conditions that guarantee the unique recovery for those examples, typically known via diverse techniques in the literature, as well as novel conditions for sparse and unsigned versions of unlabeled sensing. Similarly, our noise result also implies that the unique recovery in unlabeled sensing is locally stable.
		\end{abstract}
		
		\begin{keyword}
			homomorphic sensing, unlabeled sensing, linear regression without correspondences, real phase retrieval, mixed linear regression, algebraic geometry.
		\end{keyword}
		
	\end{frontmatter}
	
	
	\section{Introduction}
	\subsection{The homomorphic sensing property}\label{subsection:HS}
	The homomorphic sensing problem, introduced in \cite{Tsakiris-arXiv18b, Tsakiris-arXiv18b-v6} and also in the expository paper \cite{Tsakiris-ICML2019}, is posed as follows. With $\bbH$ being $\bbR$ or $\bbC$ let $\cV\subset \bbH^n$ be a linear subspace of dimension $d$ and $\cT$ a finite set of linear maps $\bbH^n\to \bbH^m$. With $v^*\in\cV$ and $\tau^*\in\cT$ we observe $y:=\tau^*(v^*)$. Given $\cV,\cT$ and $y$, then, can we \textit{uniquely} determine $v^*$ without knowing $\tau^*$? In other words, with $y$ fixed we want to know when the relations
	\begin{align*}
		y=\tau(v), \ \ \ \ \ \tau\in\cT, \  \ \ \ \ v\in \cV
	\end{align*}
	necessarily imply that $v=v^*$. This motivates the following definition.
	
	\begin{definition}[homomorphic sensing property  \cite{Tsakiris-arXiv18b, Tsakiris-arXiv18b-v6,Tsakiris-ICML2019}]\label{definition:HSP}
		Let $\cX \subset \bbH^n$ be a set of vectors and $\cT$ a finite set of linear maps $\bbH^n\to \bbH^m$. We say that $\cX$ and $\cT$ satisfy the ``homomorphic sensing property'', denoted by $\hsp(\cX,\cT)$, if the following holds:
		\begin{align*}
			\hsp(\cX,\cT):\ \ \forall  v_1,v_2\in\cX, \forall \tau_1,\tau_2\in\cT, \ \ \tau_1(v_1)=\tau_2(v_2) \Rightarrow v_1=v_2.
		\end{align*}
		If $\tau_1(v_1)=\tau_2(v_2)$ only implies $v_1=\pm v_2$, then we will use the notation $\hsp_\pm(\cX,\cT)$.
	\end{definition}
	
In this paper we study conditions under which the homomorphic sensing property is true for sets $\cX$ that are linear subspaces, subspace arrangements or collections of sparse vectors, while we will also consider the stability of the property to noise. We will develop our theory in great generality with the linear maps in $\cT$ arbitrary, but then we will specialize to specific maps of relevance in applications, as discussed in the next section. An early reference to the main results of this paper is Table \ref{table:main-results}.

\begin{table}[h!] 
\centering
	\ra{1.3}
	\begin{tabular}{ll}\toprule
		Theorem \ref{theorem:HS}& $\hsp$ for a single linear subspace $\cV$ in $\bbH^n$ and arbitrary $\cT$ \\
	Theorem \ref{theorem:HS-SA} & $\hsp$ for a subspace arrangement $\cA=(\cV_1,\dots,\cV_\ell)$ in $\bbH^n$ and arbitrary $\cT$\\		
	Theorem \ref{theorem:HS-deterministic-noise} & $\hsp$ for a single linear subspace $\cV$ in $\bbH^n$ with noisy data and arbitrary $\cT$\\
		Theorem \ref{corollary:HS-R^n} & $\hsp$ for $\bbR^n$ and $\cT$ compositions of permutations, selections and sign changes\\
		Theorem \ref{corollary:HS-k-sparse}& $\hsp$ for $k$-sparse vectors in $\bbR^n$ and $\cT$ compositions of permutations, selections and sign changes\\ 		
		\bottomrule
	\end{tabular}
	\caption{The homomorphic sensing properties studied in the paper.} \label{table:main-results}
\end{table}

\subsection{Examples of homomorphic sensing and related work}\label{subsection:Examples}

We build some insight and illustrate the significance of Definition \ref{definition:HSP} by way of some examples.

We begin with a very simple linear algebra example, where $\mathcal{T}=\{\tau_A\}$ consists of the linear transformation $\tau_A: \bbH^n \rightarrow \bbH^m$ induced by multiplication with an $m \times n$ matrix $A$. Then $\hsp(\bbH^n, \{\tau_A\})$ holds if and only if $\tau_A$ is injective, that is, if and only if $A$ has full column rank. This is true for a generic $A$ as soon as $m \ge n$. On the other hand, $\hsp(\cX, \{\tau_A\})$ holds if and only if no two vectors in $\cX$ differ by a nonzero nullvector of $A$. 

Our next example is compressed sensing. Denoting by $\ocKI$ the set of $k$-sparse vectors of $\bbH^n$ (the choice of the symbol $\ocKI$ will be made clear in \S \ref{subsection:HS-applications}), unique recovery of $k$-sparse vectors under an $m \times n$ sensing matrix $A$ is equivalent to $\hsp\big(\ocKI, \{\tau_A\}\big)$. This is true if and only if any $\min\{2k,n\}$ columns of $A$ are linearly independent, a property that is satisfied by a generic $A$ as soon as $m \ge \min\{2k,n\}$. 
		
A more interesting example is \textit{phase retrieval}, which dates back to the 1910's, when the research on \textit{X-ray crystallography} was launched; see \cite{Grohs-SIAM-Review2020} for a vivid account. Depending on whether one works over $\bbR$ or $\bbC$ one has \emph{real} or \emph{complex phase retrieval}. In the real phase retrieval problem, we are given an $m$-dimensional vector $y$ whose entries are the magnitudes of the linear measurements $Ax^*$, where $x^*$ is some unknown $n$-dimensional vector to be recovered. Equivalently, with $\cB_m$ the set of $m\times m$ sign matrices, i.e., diagonal matrices with $\pm 1$ on the diagonal, we are given $y=B^*Ax^*$ for some unknown sign matrix $B^*\in\cB_m$, and the goal is to recover $x^*$ from $y$. Since uniquely recovering a nonzero $x^*$ is impossible, we consider unique recovery of $x^*$ up to sign. In other words, with $\cB_{m}A:=\{BA:B\in\cB_{m} \}\subset\bbR^{m\times n}$ we consider  $\hsp_{\pm}(\bbR^n, \cB_m A)$, where $\cB_m A$ is identified with our finite set $\cT$ of linear maps $\bbR^n \rightarrow \bbR^m$. In 2006, it was proved by \cite{Balan-ACHA2006} in a frame-theoretical language that $m\geq 2n-1$ suffices for a generic $A\in\bbR^{m\times n}$ to enjoy $\hsp_\pm(\bbR^n,\cB_m A)$, and this is necessary for any $A\in\bbR^{m\times n}$. If $x^* \in \ocKI$ is $k$-sparse, a situation considered in \textit{sparse real phase retrieval} \cite{Wang-ACHA2014}, then \cite{Wang-ACHA2014} and \cite{Akccakaya-arXiv2013v2} have independently shown that, for $A\in\bbR^{m\times n}$ generic, the property $\hsp_\pm(\ocKI, \cB_m A)$ is equivalent to $m\geq \min\{2n-1,2k \}$. These results also hold for the problem of \textit{symmetric mixture of two linear regressions} \cite{Balakrishnan-AoS2017}, since it bears the same formulation as real phase retrieval; see \cite{Chen-MP2019}, \cite{Klusowski-TIT2019} for discussions that connect the two problems. \emph{Complex phase retrieval} can also be formulated in terms of the homomorphic sensing property. With reference to Definition \ref{definition:HSP}, in complex phrase retrieval $n=\ell^2$ and the set $\cX$ of interest is the set of $\ell \times \ell$ rank-$1$ Hermitian matrices, while $\cT$ consists of a single linear transformation $\tau_A: \bbC^n \rightarrow \bbC^m$ that takes $X \in \bbC^{\ell \times \ell}$ to $\tau_A(X) = (\langle A_1, X\rangle, \dots, \langle A_m, X\rangle)$, with each $A_i$ an $\ell \times \ell$ measurement matrix of rank $1$. \cite{Bandeira-ACHA2014} famously conjectured that i) unique recovery is impossible under less than $4\ell-4$ measurements, while ii) at least $m=4\ell-4$ measurements with $A_i$ generic suffice for $\hsp(\cX,\{\tau\})$ to be true. The first part of the conjecture was disproved in its generality in \cite{Vinzant-SampTA2015}, where a counterexample was given for the case $m=4$, while the second part was proved in \cite{Conca-ACHA2015}. 

More general than phase retrieval is matrix recovery, where one aims for unique identifiability of a matrix of bounded rank from a set of linear measurements. This problem amounts to checking $\hsp(\mathcal{M}_\bbH(r, p \times q),\{\tau_A\})$, where $\mathcal{M}_\bbH(r, p \times q) \subset \bbH^{p \times q}$ is the algebraic variety of $p \times q$ matrices of rank at most $r$ over $\bbH$ and $\tau_A = (\langle A_1, \cdot \rangle, \dots, \langle A_m, \cdot \rangle)$. For a discussion of this line of work and generalizations to arbitrary algebraic varieties we refer the reader to \cite{Cai-2018}. An essential difference though of this family of problems with the homomorphic sensing framework is that they are only a very special case of it, since there is a single linear transformation $\tau_A$ involved. Indeed, having multiple transformations and not knowing which transformation the available data came from adds an entire level of complexity to the problem. 
	
Another line of research, which can be cast within the homomorphic sensing framework and in fact inspired it, has its origin in statistics in the context of \textit{record linkage} \cite{Fellegi-1969} and the \textit{broken sample problem} \cite{Degroot-AoS1980}; see \cite{Slawski-JoS19} for detailed discussions. Recently, interest was revived by \cite{Unnikrishnan-Allerton2015,Unnikrishnan-TIT18}, where the problem was studied abstractly under the name \textit{unlabeled sensing}. Simply stated, unlabeled sensing is the problem of solving a linear system of equations for which the right-hand-side vector has undergone an unknown permutation and some of its entries have been discarded. This was also recently and independently considered by \cite{Han-ACHA2018}. The special case where no entries are discarded is known in subsequent work as \textit{linear regression without correspondences} \cite{Hsu-NIPS17,Pananjady-TIT18,Haghighatshoar-TSP18,Slawski-arXiv2019b,Slawski-UAI2019,Dokmanic-SPL2019,Zhang-arXiv2019v2,Tsakiris-TIT2020,Peng-SPL2020,Wang-arXiv2020v2}. With $A$ an $m \times n$ matrix as above and $\cS_{r,m}$ the set of all $r \times m$ rank-$r$ selection matrices, i.e., matrices whose rows are formed by $r$ distinct standard basis vectors of $\bbR^m$, unique recovery in unlabeled sensing is equivalent to $\hsp(\bbR^n,\cS_{r,m}A)$. Via different combinatorial techniques \cite{Unnikrishnan-Allerton2015,Unnikrishnan-TIT18} and \cite{Han-ACHA2018} proved that $r\geq 2n$ suffices to guarantee $\hsp(\bbR^n,\cS_{r,m}A)$ for $A\in \bbR^{m\times n}$ generic. For the converse, \cite{Han-ACHA2018} proved that $r\geq 2n-1$ is necessary for $\hsp(\bbR^n,\cS_{r,m}A)$ and \cite{Unnikrishnan-Allerton2015,Unnikrishnan-TIT18} proved that, if $m$ is odd with $m=r$ and $n\geq 2$, then $r\geq 2n$ is necessary. 
	
A combination of real phase retrieval and unlabeled sensing, which we refer to as \textit{unsigned unlabeled sensing}, was explored in \cite{Lv-ACM2018}, where $\cT$ is $\cS_{r,m}\cB_m:=\{SB:S\in\cS_{r,m},B\in\cB_m \}$ and the interest is in $\hsp_\pm(\bbR^n, \cS_{r,m}\cB_m A)$. By extending the approach of \cite{Han-ACHA2018}, it was established in \cite{Lv-ACM2018} that $r\geq 2n$ is sufficient for $\hsp_\pm(\bbR^n,\cS_{r,m}\cB_m A)$ for $A\in\bbR^{m\times n}$ generic and this is necessary if $n\geq 2$.

As our last example, we briefly mention the very recent work of \emph{unlabeled principal component analysis} \cite{yao2021unlabeled}. This is a generalization of linear regression without correspondences to the algebraic variety $\cX=\mathcal{M}_\bbH(r, p \times q) \subset \bbH^{p \times q}$ of $p \times q$ matrices or rank at most $r$ over $\bbH$, with $\cT$ being permutations acting on the matrix entries. The property of interest here is $\hsp(\mathcal{M}_\bbH(r, p \times q),\cT)$ up to a permutation of the rows or columns of the matrix. This is established in \cite{yao2021unlabeled} for a generic $X \in\mathcal{M}_\bbH(r, p \times q)$.

\subsection{Existing homomorphic sensing theory}\label{subsection:existing-HS}

Prior work on the abstract theory of homomorphic sensing has only considered Definition \ref{definition:HSP} for $\cX=\cV$ a linear subspace of $\bbH^n$ of dimension $d$. Even though this formulation is linear algebraic, its nature is inherently algebraic-geometric, because whenever $\tau_1(v_1) = \tau_2(v_2)$, we have that $v_1,v_2$ satisfy the quadratic relation $\tau_1(v_1) \wedge \tau_2(v_2)=0$, with $\wedge$ being the exterior product. This was the main insight of \cite{Tsakiris-arXiv18b, Tsakiris-arXiv18b-v6, Tsakiris-ICML2019} leading to the following results. With linear maps $\tau_1,\tau_2: \bbH^n \rightarrow \bbH^m$  let $\otau_1,\otau_2:\bbC^n \rightarrow \bbC^m$ be their complexifications and let $T_1,T_2$ be their matrix representations with respect to the standard basis. Let $\rho$ be a linear projection onto the image $\im(\tau_2)$ of $\tau_2$ with matrix representation $P$ and complexification $\orho$. With $w$ a vector of variables, the $2\times 2$ minors of the $m \times 2$ matrix  $[PT_1w\, \, \,  T_2 w]$ are polynomials in entries of $w$, so their vanishing locus in $\bbC^n$ is a complex algebraic variety, say $\cY_{\rho\tau_1,\tau_2}$. Removing from $\cY_{\rho\tau_1,\tau_2}$ the union of linear subspaces $\cZ_{\rho\tau_1,\tau_2}:=\ker(\orho\otau_1-\otau_2)\cup\ker(\orho\otau_1)\cup\ker(\otau_2)$ gives 
		\begin{align*}
		\cU_{\rho\tau_1,\tau_2}=\cY_{\rho\tau_1,\tau_2}\backslash \cZ_{\rho\tau_1,\tau_2}
	\end{align*} which is an open set in the Zariski topology of $\cY_{\rho\tau_1,\tau_2}$, also called quasi-variety. The following was proved in \cite{Tsakiris-arXiv18b,Tsakiris-arXiv18b-v6}: If for any $\tau_1,\tau_2\in\cT$ it holds that $\rank(\tau_1):=\rank(T_1)\geq 2d$ and $\rank(\tau_2)\geq 2d$, and if there exists a linear projection $\rho$ onto $\im(\tau_2)$ such that $d \le \codim(\cU_{\rho\tau_1,\tau_2})$, then a generic subspace $\cV\subset \bbH^n$ of dimension $d$ satisfies $\hsp(\cV,\cT)$; here generic is meant in the sense that $\hsp(\cV,\cT)$ is true for every $\cV$ in a dense open set of the Grassmannian $\Gr_{\bbH}(d,n)$. In the same work, the dimension of the quasi-variety $\cU_{\rho\tau_1,\tau_2}$ was calculated for the case of unlabeled sensing, leading to the same sufficient conditions as in \cite{Unnikrishnan-TIT18} and \cite{Han-ACHA2018}, mentioned above. Moreover, an \emph{almost everywhere} $\hsp(\cV,\cT)$ type of result was proved: for a generic $\cV$, as long as no two maps in $\cT$ are multiples of each other and $\rank(\tau) \ge d+1$ for every $\tau \in \cT$, then $\hsp(\cU,\cT)$ holds true, where $\cU$ is a Zariski dense open set of $\cV$. This is a much easier result to obtain than the one asserting the homomorphic sensing property on the entire space $\cV$.\footnote{For \emph{almost everywhere} type of results in matrix recovery and phase retrieval see \cite{Rong-ACHA2019} and \cite{Huang-ACHA2021}, respectively.}
	
\section{Main Results}\label{section:main-results}

In this section we discuss our main results, while proofs are in \S \ref{section:proofs} and \S \ref{section:Appendix}. In \S \ref{subsection:HSAS} we give new improved results for $\cX=\cV$ a linear subspace of $\bbH^n$. In \S \ref{subsection:HSSA} we generalize the picture to $\cX= \cA=(\cV_1,\dots,\cV_\ell)$ a subspace arrangement of $\bbH^n$. In \S \ref{subsection:noisy-HS} we treat the case of a linear subspace in the presence of noise. In \S \ref{subsection:HS-applications} we specialize our results to unlabeled sensing and phase retrieval and we obtain via a unified framework either results that are already known via diverse techniques in the literature or entirely new results. Table \ref{table:main-results} gives a summary.

\subsection{Homomorphic sensing of a linear subspace} 	\label{subsection:HSAS}

Recall from \S \ref{subsection:existing-HS} that $d \le \codim(\cU_{\rho\tau_1,\tau_2})$ and $\rank(\tau) \ge 2d$ for every $\tau_1,\tau_2,\tau \in \cT$, is a known sufficient condition for $\hsp(\cV,\cT)$, where $\cV$ is a generic linear subspace of dimension $d$. A suboptimal feature of this result is the presence of the projection $\rho$. This is an artifact of the proof technique in \cite{Tsakiris-arXiv18b, Tsakiris-arXiv18b-v6}, which involves a projection onto $\im(\tau_2)$. One of our main contributions in this paper is to dispense with $\rho$ by following an entirely different technique that uses filtrations. Consider the complex algebraic variety $\cY_{\tau_1,\tau_2}$ defined by the vanishing of the $2\times 2$ minors of $[T_1w\, \, \,  T_2 w]$, the union of linear subspaces $\cZ_{\tau_1,\tau_2}:=\ker(\otau_1-\otau_2)\cup\ker(\otau_1)\cup\ker(\otau_2)$ and the quasi-variety
	\begin{align*}
		\cU_{\tau_1,\tau_2}=\cY_{\tau_1,\tau_2}\backslash \cZ_{\tau_1,\tau_2}.
	\end{align*}
	
	We have the following improvement over \cite{Tsakiris-arXiv18b, Tsakiris-arXiv18b-v6,Tsakiris-ICML2019}: 
	
	\begin{theorem}[homomorphic sensing of a single subspace] \label{theorem:HS}
		Suppose $\rank(\tau)\geq 2d$ for every $\tau\in\cT$. Then $\hsp(\cV,\cT)$ holds true for a generic subspace $\cV$ of $\bbH^n$ of dimension $d$ whenever 
		\begin{align}\label{eq:HS-condition_U}
			d \leq \codim(\cU_{\tau_1,\tau_2}),\ \ \ \forall \tau_1,\tau_2\in\cT.
		\end{align}
	\end{theorem}
	
	Note that, by definition, $\cU_{\tau_1,\tau_2}$ is a subset of $\cU_{\rho\tau_1,\tau_2}$, so condition \eqref{eq:HS-condition_U} is tighter than that of \cite{Tsakiris-arXiv18b,Tsakiris-arXiv18b-v6,Tsakiris-ICML2019}. Indeed, condition \eqref{eq:HS-condition_U} is the tightest possible in the following sense.

	\begin{prop}\label{prop:HS}
		Suppose $\bbH=\bbC$ and that \eqref{eq:HS-condition_U} is not true. Then $\hsp(\cV,\cT)$ is violated for a generic subspace $\cV\subset\bbH^n$ of dimension $d$.
	\end{prop}

	Using the proof technique of Theorem \ref{prop:HS}, we get the following extension for $\hsp_{\pm}(\cV,\cT)$.
	\begin{prop}\label{prop:HS-sign}
		Suppose that for every $\tau\in\cT$ we have $\rank(\tau)\geq 2d$. Let $\cU^\pm_{\tau_1,\tau_2}:=\cU_{\tau_1,\tau_2}\backslash \ker(\otau_1 +\otau_2)$. Then $\hsp_\pm(\cV,\cT)$ holds true for a generic subspace $\cV$ of $\bbH^n$ of dimension $d$ whenever 
		\begin{align}\label{eq:HS-condition_U_sign}
			d \le \codim(\cU^\pm_{\tau_1,\tau_2}),\ \ \ \forall \tau_1,\tau_2\in\cT.
		\end{align}	
	\end{prop}

	\subsection{Homomorphic sensing of a subspace arrangement}\label{subsection:HSSA}
	
	We extend Theorem \ref{theorem:HS} from a single subspace $\cV$ to a subspace arrangement $\cA=(\cV_1,\dots,\cV_\ell)$, the latter being an ordered set of subspaces $\cV_i, \, i \in [\ell]:=\{1,\dots,\ell\}$ of $\bbH^n$. With $d_i=\dim (\cV_i)$, we refer to $(d_1,\dots, d_\ell)$ as the dimension configuration of $\cA$. Thus, by a \textit{generic} subspace arrangement $\cA$ with dimension configuration $(d_1,\dots, d_\ell)$ we mean a non-empty Zariski open subset of the product $\Gr_\bbH(d_1,n)\times \cdots\times \Gr_{\bbH}(d_\ell,n)$ of Grassmannians (see also \S \ref{section:pre}). Consider an ordered set $\rI=(\cI_1,\dots,\cI_s)$ of subsets of $[\ell]$. Each $\cI_j$ gives rise to a subspace $\cV_{\cI_j}:=\sum_{i\in\cI_j}\cV_i$  with dimension upper bounded by $d_{\cI_j}:=\sum_{i\in\cI_j}d_i$, where $\cV_{\varnothing}:=0$. Thus the ordered set $\rI$, together with $\cA$, induces the \textit{structured} subspace arrangement $\cA_{\rI}=(\cV_{\cI_1},\dots,\cV_{\cI_s})$. This construction allows various levels of flexibility that will be exploited later in the paper. For example, $\cA_{\rI}$ becomes the original $\cA$ when $\cI_j=\{j\}$ and $s=\ell$, and if in addition $s=1$, then $\cA_{\rI}$ becomes a single subspace. We write $\ocAI:=\bigcup_{j\in[s]}\cV_{\cI_j}$ and consider the property $\hsp(\ocAI,\cT)$. We have:
	\begin{theorem}[homomorphic sensing of a subspace arrangement] \label{theorem:HS-SA}
		Let $(d_1,\dots,d_\ell)$ be a dimension configuration and $\rI=(\cI_1,\dots,\cI_s)$ an ordered set of subsets of $[\ell]$. Let $d:=\max_{j\in[s]} d_{\cI_j}$. Suppose $\rank(\tau)\geq 2d$ for every $\tau\in\cT$. Then $\hsp(\ocAI,\cT)$ holds for a generic subspace arrangement $\cA=(\cV_1,\dots,\cV_{\ell})$ with $\dim(\cV_i)=d_i$, whenever \eqref{eq:HS-condition_U} holds. Similarly, $\hsp_\pm(\ocAI,\cT)$ holds for a generic subspace arrangement $(\cV_1,\dots,\cV_{\ell})$ with $\dim(\cV_i)=d_i$, whenever \eqref{eq:HS-condition_U_sign} holds.
	\end{theorem}

\begin{remark}
Affine subspaces are very important in applications, because data tend to come biased away from the origin; see for example the case of affine real phase retrieval \cite{Gao-AAM2018} or arrangements of affine subspaces in clustering \cite{Tsakiris-TPAMI2018}. 
Given an affine subspace $\cV+b$ of $\bbH^n$ one obtains immediately a result about $\hsp(\cV+b,\cT)$ by applying Theorem \ref{eq:HS-condition_U} to the $(d+1)$-dimensional linear subspace $\cV+\Span(b)$. One obtains a homomorphic sensing result for an arrangement of affine subspaces in a similar manner by applying Theorem \ref{theorem:HS-SA} to $(\cV_1 + b_1,\dots,\cV_\ell+b_\ell)$. 
An approach which takes into explicit consideration the nature of the affine subspaces will be considered in a future manuscript.   
\end{remark} 
	
	\subsection{Noisy homomorphic sensing}\label{subsection:noisy-HS}
		For any (column) vector $q\in\bbH^m$, denote by $q^H$ its Hermitian transpose. With $u, w \in \bbH^m$ define $\inner{u}{w}:=u^Hw$, the standard inner product $\bbH^m\times \bbH^m\to \bbH$, and also $\norm{w}:=\sqrt{\inner{w}{w}}=\sqrt{w^Hw}$. 
	
	We consider the homomorphic sensing problem in the presence of additive noise $\epsilon\in\bbH^{m}$. For $v^*\in\cV$ and $\tau^*\in\cT$ set $y=\tau^*(v^*)$ and $\oy=y+\epsilon$. We are interested in the optimization problem
	\begin{align}\label{eq:HS-MLE}
		(\htau,\hv)\in\argmin_{v\in\cV,\tau\in\cT} \norm{\oy-\tau(v)}.
	\end{align}  
	What can we say about the optimal solution $\hv$? Under what conditions is $\hv$ close to $v^*$?
		
	For a non-trivial subspace $\cW\subset \bbH^m$ and a non-zero $u\in\bbH^m$ define
	\begin{align*}
		\cos(u,\cW):=\max\big\{ \frac{\inner{u}{w}+\inner{w}{u}}{2\norm{u}}:
		w\in\cW \text{ and } \norm{w}=1 \big\}.
	\end{align*}
	
	Denote by $\sigma(X)$ the largest singular value of a matrix $X$. Then we have the following stability result.
	
	\begin{theorem}[noisy homomorphic sensing] \label{theorem:HS-deterministic-noise}
Let $\cV\subset\bbH^{n}$  be a subspace of dimension $d$ that satisfies $\hsp(\cV,\cT)$ and let $V\in\bbH^{n\times d}$ be a matrix that has $\cV$ as its column-space. Let $(\htau,\hv)$ be a solution to \eqref{eq:HS-MLE} with $\hT$ the matrix representation of $\htau$. Set $\cT_1:=\{\tau\in\cT:y\in\tau(\cV)\}$. If $\cT = \cT_1$ or 
		\begin{align}\label{eq:deterministic_noise_condition}
			2\norm{\epsilon}<\norm{y}\Big(1-\max_{\tau\in\cT\backslash\cT_1}\cos\big(y,\tau(\cV)\big)\Big),	
		\end{align} 
		then $\hv - v^*=V(\hT V)^{\dagger}\epsilon$, where $(\hT V)^{\dagger}$ is the pseudoinverse of $\hT V$. In particular $\|\hv - v^*\|_2 \leq \sigma(V(\hT V)^{\dagger}) \|\epsilon\|_2$.
	\end{theorem}    
	
\subsection{Applications of homomorphic sensing theory}\label{subsection:HS-applications} 
We consider the applications of Theorems \ref{theorem:HS}-\ref{theorem:HS-deterministic-noise} to problems mentioned in \S \ref{subsection:Examples}, namely linear regression without correspondences ($\cS_{m}$), unlabeled sensing ($\cS_{r,m}$), real phase retrieval ($\cB_m$) and unsigned unlabeled sensing ($\cS_{r,m}\cB_m$). Bounding the dimension of $\cU_{\tau_1,\tau_2}$ for each of these cases and applying Theorem \ref{theorem:HS}, gives us the following results, which have already been obtained in a diverse literature via diverse methods:
 
	\begin{theorem}[unlabeled sensing and real phase retrieval]\label{corollary:HS-R^n}
		For a generic matrix $A$ of $\bbR^{m\times n}$
		\begin{enumerate}[label=\roman*)]
			\item $m\geq 2n\Rightarrow\hsp(\bbR^n,\cS_m A)$ \cite{Tsakiris-ICML2019,Unnikrishnan-TIT18,Han-ACHA2018,Dokmanic-SPL2019}. 
			\item  $r\geq 2n\Rightarrow\hsp(\bbR^n,\cS_{r,m} A)$  \cite{Tsakiris-ICML2019,Unnikrishnan-TIT18,Han-ACHA2018}.
			\item  $m\geq 2n\Rightarrow \hsp_{\pm}(\bbR^n,\cB_m A)$ \cite{Balan-ACHA2006,Dokmanic-SPL2019}.
			\item\label{C1.4} $r\geq 2n\Rightarrow\hsp_{\pm}(\bbR^n,\cS_{r,m}\cB_m A)$ \cite{Tsakiris-ICML2019,Lv-ACM2018}.
		\end{enumerate}	
	\end{theorem}
	
	Next we consider the sparse counterpart of Theorem \ref{corollary:HS-R^n}. This is mostly unexplored territory in prior work and the main player here is Theorem \ref{theorem:HS-SA}. Consider the standard basis $e_1,\dots,e_n$ of $\bbR^n$ and the subspace arrangement $\cK=(\cV_1,\dots,\cV_n)$ of $\bbR^n$ with $\cV_i=\Span(e_i)$. Let $s=\binom{n}{k}$ and consider the set $\rI=(\cI_1,\dots,\cI_s)$ of all subsets of $[n]$ of cardinality $k$, ordered, say, in the lexicographic order. It gives the structured subspace arrangement $\cK_{\rI}=(\cV_{\cI_1},\dots,\cV_{\cI_s})$ where $\cV_{\cI_j}=\sum_{i\in\cI_j}\cV_i$. By construction, the union of subspaces $\ocKI:=\cup_{j\in[s]}\cV_{\cI_j}$ is the set of all $k$-sparse vectors of $\bbR^n$. With this notation, sparse real phase retrieval is equivalent to $\hsp_{\pm}(\ocKI,\cB_m A)$, and this has been studied by \cite{Wang-ACHA2014} and \cite{Akccakaya-arXiv2013v2}, while sparse unlabeled sensing ($\hsp(\ocKI,\cS_{r,m} A)$) and sparse unsigned unlabeled sensing ($\hsp_{\pm}(\ocKI,\cS_{r,m}\cB_m A)$) have not been considered yet, to the best of our knowledge. Theorem \ref{theorem:HS-SA} together with the bounds on $\dim (\cU_{\tau_1,\tau_2})$ give:
	
	\begin{theorem}[sparse unlabeled sensing and sparse real phase retrieval]\label{corollary:HS-k-sparse}
		For a generic matrix $A$ of $\bbR^{m\times n}$ and $k\leq n$
		\begin{enumerate}[label=\roman*)]
			\item\label{C2.1} $m\geq 2k\Rightarrow\hsp(\ocKI,\cS_m A)$.
			\item\label{C2.2}  $r\geq 2k\Rightarrow\hsp(\ocKI,\cS_{r,m} A)$.
			\item\label{C2.3} $m\geq 2k\Rightarrow \hsp_{\pm}(\ocKI,\cB_m A)$ \cite{Wang-ACHA2014,Akccakaya-arXiv2013v2}.
			\item\label{C2.4} $r\geq 2k\Rightarrow\hsp_{\pm}(\ocKI,\cS_{r,m}\cB_m A)$.
		\end{enumerate}
	\end{theorem}
	
	Our final result is a corollary of Theorem \ref{theorem:HS-deterministic-noise}. We only state the result for unlabeled sensing, where $y=S^*Ax^*$ for some $S^*\in\cS_{r,m}$, $\oy=y+\epsilon$, and the objective function of interest as a special case of \eqref{eq:HS-MLE} is 
	\begin{align*}
		(\hS,\hx)\in\argmin_{x\in\bbR^{n},\, S\in\cS_{r,m}} \norm{\oy-SAx}.
	\end{align*} 
	\begin{corollary}\label{corollary:noise}
	    \vspace{-0.4cm}
		If \eqref{eq:deterministic_noise_condition} holds with $\cT=\cS_{r,m}$ and if $A$ satisfies $\hsp(\bbR^n,\cS_{r,m}A)$, then $\hx-x^*=(\hT A)^{\dagger}\epsilon$.
	\end{corollary}
	We note that condition \eqref{eq:deterministic_noise_condition} of Corollary \ref{corollary:noise} defines a non-asymptotic regime, where the local stability of estimating $x^*$ is guaranteed, and this implies the asymptotic result of \cite{Unnikrishnan-TIT18}.
	
	%
	
	
	\section{Preliminaries}\label{section:pre}
Let $\bbH$ be equal to $\bbR$ or $\bbC$. For $j=1,2$ let $\tau_j$ be an $\bbH$-linear map $\bbH^n\to\bbH^m$ and write $T_j \in \bbH^{m \times n}$ for its matrix representation with respect to the canonical basis. Denote by $\otau_j:\bbC^n\to\bbC^m$ the complexification of $\tau_j$. That is $\otau_j:=\tau_j$ if $\bbH=\bbC$, and $\otau_j(u+iv):=\tau_j(u)+i\tau_j(v)$ for every $u,v\in\bbR^n$ if $\bbH=\bbR$; here $i = \sqrt{-1}$. Note that if $\bbH=\bbR$, then $T_j$ is also a matrix representation for $\otau_j$. With $\lambda\in\bbC$, denote by $\cE_{(\tau_1,\tau_2),\lambda}$ the set of all $w \in \bbC^n$ satisfying $\otau_1(w)=\lambda \otau_2(w)$. This is a $\bbC$-subspace of $\bbC^n$. If $\lambda \in \bbR$, then $\cE_{(\tau_1,\tau_2),\lambda}\cap\bbR^m$ is an $\bbR$-subspace of $\bbR^n$ and we have $\dim_\bbR \big(\cE_{(\tau_1,\tau_2),\lambda}\cap\bbR^n \big)=\dim_\bbC\big(\cE_{(\tau_1,\tau_2),\lambda}\big)$, where $\dim_\bbR, \dim_\bbC$ denote real and complex vector space dimension respectively. In the sequel, we will drop the subscript indicating the field, with the convention that by $\dim (\cW)$ we mean $\dim_\bbH (\cW)$ whenever $\cW$ is a $\bbH$-subspace, while $\cV \cap \cE_{(\tau_1,\tau_2),\lambda}$ will always be treated as an $\bbR$-subspace whenever $\cV$ is such. For simplicity, we write $\cE_{\tau_j,\lambda}:=\cE_{(\tau_j,\id),\lambda}$ for the eigenspace of $\tau_j$ corresponding to eigenvalue $\lambda$, where $\id$ is the identity map. For a map of sets $\tau:\cX\to \cY$ we let $\tau^{-1}(Q)$ be the inverse image of $\cQ\subset \cY$ under $\tau$. Denote by $0$ the trivial subspace, the zero vector, and the number zero, to be made clear by the context. We say that two subspaces $\cV,\cW$ do not intersect if $\cV \cap \cW=0$.

An algebraic variety is a subset of $\bbH^n$ defined as the common zero locus of a set of polynomials in $n$ variables with coefficients in $\bbH$. The \textit{Zariski topology} on $\bbH^n$ is defined by identifying closed sets with algebraic varieties of $\bbH^n$. Hence Zariski open sets arise as loci in $\bbH^n$ of non-simultaneous vanishing of sets of polynomials. An irreducible algebraic variety is one which can not be written as the union of two proper subvarieties of it. Here by subvariety we mean a closed set in the subspace topology. By a \textit{generic}  point of an irreducible algebraic variety having some property of interest, we mean that there is a non-empty Zariski open (and thus necessarily dense) subset in the variety, each element of which satisfies the property. We denote by $\Gr_\bbH(d,n)$ the Grassmannian of $d$-dimensional $\bbH$-subspaces of $\bbH^n$. One defines a Zariski topology in projective space in a similar fashion as above and under the Pl\"ucker embedding \cite{Harris-AG} $\Gr_\bbH(d,n)$ becomes an irreducible projective variety of dimension $d(n-d)$. For integers $1 \le d_1,\dots,d_\ell \le n-1$ the product $\Gr_\bbH(d_1,n) \times \cdots \times \Gr_\bbH(d_\ell,n)$ is also an irreducible projective variety. Since the affine space $\bbH^{m\times n}$ is irreducible as well, we have justified what we mean by a generic $m\times n$ matrix over $\bbH$ or a generic $d$-dimensional $\bbH$-subspace of $\bbH^n$ or a generic subspace arrangement $(\cV_1,\dots,\cV_\ell)$ with $\cV_i  \in \Gr_\bbH(d_i,n)$. Another classical irreducible variety that will play a role is the \textit{flag variety} $\F_{\bbH}(d_0,d,n)$. This lives in the product $\Gr_{\bbH}(d_0,n)\times \Gr_{\bbH}(d,n)$ and consists of those pairs $(\cV_0,\cV)$ that satisfy $\cV_0\subset \cV$. The following fact about projections, proved in \S \ref{proof:lem:flag}, will be needed in the proof of Theorem \ref{theorem:HS}:

\begin{lemma} \label{lem:flag}
Let $\phi: \F_{\bbH}(d_0,d,n) \rightarrow \Gr_{\bbH}(d,n)$ be the canonical projection that sends $(\cV_0,\cV)$ to $\cV$. If $\sU$ is a non-empty Zariski open subset of $\F_{\bbH}(d_0,d,n)$, then the image $\phi(\sU)$ contains a non-empty Zariski open subset of $\Gr_{\bbH}(d,n)$.
\end{lemma} 

The dimension $\dim(\cQ)$ of an algebraic variety $\cQ$ is the maximal length $t$ of the chains $\cQ_0\subset\cQ_1\subset\cdots\subset \cQ_t$ of distinct irreducible algebraic varieties contained in $\cQ$. The dimension of any set $\cQ$ is the dimension of its \textit{closure} $\cQ^{\cl}$, i.e., $\cQ^{\cl}$ is the smallest algebraic variety which contains $\cQ$. Linear subspaces are algebraic varieties and their linear algebra dimension coincides with their algebraic-geometric dimension. By convention $\dim \cQ = -1$ if and only if $\cQ$ is empty, while over $\bbC$ we have that $\dim (\cQ)=0$ if and only if $\cQ$ is a finite set of points. A polynomial $p$ is called homogeneous if, writing $p$ as a linear combination of distinct monomials, all monomials that appear with non-zero coefficient have the same degree. For instance, $\cY_{\tau_1,\tau_2}$ is clearly defined by homogeneous polynomials, while so is any union of linear subspaces such as $\cZ_{\tau_1,\tau_2}$. Let $\cY, \cZ$ be two algebraic varieties defined by homogeneous polynomials and set $\cU = \cY \setminus \cZ$. Then $\cU^{\cl}$ is also defined by homogeneous polynomials \cite{cox2013ideals}. The next statement is a folklore fact in commutative algebra and algebraic geometry and we will use it often:

\begin{lemma}\label{lemma:subspaces_intersect_subspace}
Given algebraic varieties $\cQ_1,\dots,\cQ_t$ in $\bbC^n$ of dimensions $r_1,\dots,r_t$, each defined by homogeneous polynomials, there exists a $\bbC$-subspace $\cV \in \Gr_{\bbC}(d,n)$ with $\dim(\cQ_j\cap\cV)=\max\{r_j+d-n,0\}$ for any $j\in[t]$.  
\end{lemma}

Another fact that will play a role in the proof of Proposition \ref{prop:HS}, proved in \S \ref{proof:lemma:homogeneous}, is: 

\begin{lemma}\label{lemma:homogeneous}
Let $0 \subsetneq \cZ\subset \cY$ be two algebraic varieties of $\bbC^n$ defined by homogeneous polynomials. If $\dim(\cY\backslash \cZ)>n-d$, then a generic subspace $\cV\subset \bbC^n$ of dimension $d$ intersects $\cY\backslash \cZ$.
\end{lemma}

Next is an important fact, Lemma 5 of \cite{Tsakiris-arXiv18b-v6}, used in the proof of Theorem \ref{theorem:HS}:

\begin{lemma}\label{lemma:lemma5ofHS}
Let $n,d$ be positive integers with $n \ge 2d$. Let $\tau: \bbC^n \rightarrow \bbC^n$ be a $\bbC$-linear map with $\dim(\cE_{\tau,\lambda})\leq n-d$ for every $\lambda \in \bbC$. Then there is a $\bbC$-subspace $\cV$ of $\bbC^n$ of dimension $d$ such that $\dim(\cV+\tau(\cV))=2d$.
\end{lemma}

We close with some facts needed for the proof of Theorems \ref{corollary:HS-R^n} and \ref{corollary:HS-k-sparse}. With $S_1,S_2\in\cS_{r,m}$, it was proved in \cite{Tsakiris-arXiv18b-v6} that there is a projection $P$ onto the column space of $S_2$ with $\dim(\cU_{PT_1,T_2})\leq m-\floor{r/2}$. Since $\cU_{S_1,S_2}\subset\cU_{PS_1,S_2}$, we have $\dim(\cU_{S_1,S_2}) \le m-\floor{r/2}$. It is also easy to see that $\cU_{B_1,B_2}=\varnothing$. We have:

\begin{lemma}\label{lemma:U_dimensionbound}
		Let $\Pi_1,\Pi_2\in\cS_m$, $S_1,S_2\in\cS_{r,m}$, and $B_1,B_2\in\cB_m$ be permutations, rank-$r$ selections,  and sign matrices, respectively.
		\begin{enumerate}[label=\roman*)]
			\item $m\geq 2n\Rightarrow \dim(\cU_{\Pi_1,\Pi_2})\leq m-n$.
			\item $r\geq 2n\Rightarrow\dim(\cU_{S_1,S_2})\leq m-n$.
			\item $m\geq 2n\Rightarrow \dim(\cU_{B_1,B_2}^{\pm})\leq m-n$.
			\item $r\geq 2n\Rightarrow \dim(\cU_{S_1B_1,S_2B_2}^{\pm})\leq m-n$.
		\end{enumerate}
	\end{lemma}

	\section{Proofs}\label{section:proofs}	

In this section we give the proofs. For the convenience of the reader a dependency graph of the various statements is given in Figure \ref{fig:TheoremsDependency}.
	
	\begin{figure}
	    \centering
	    \includegraphics[width=0.6\textwidth]{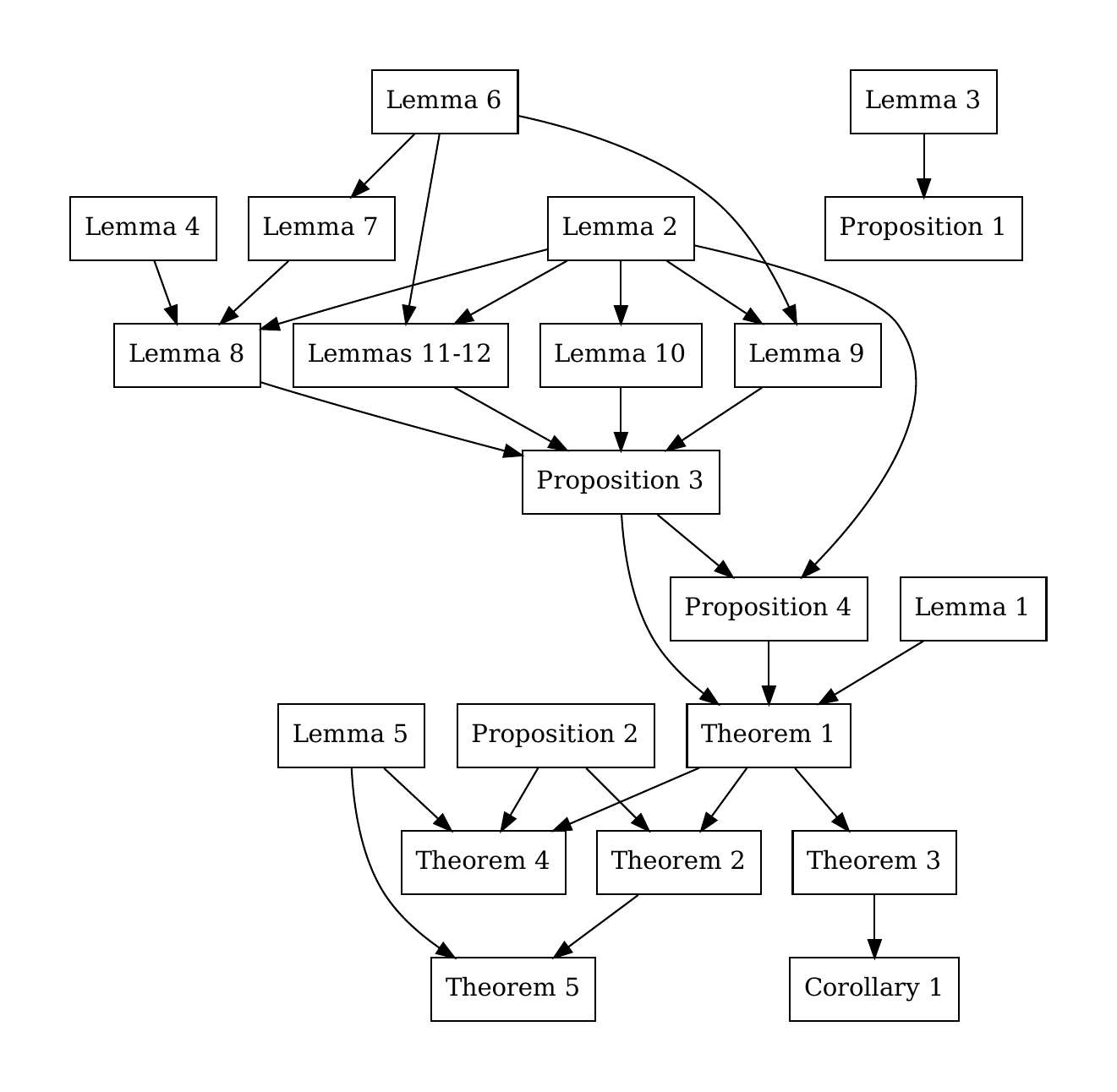}
	    \caption{Dependency graph of the various statements in the paper.}
	    \label{fig:TheoremsDependency}
	\end{figure}

\subsection{Proof of Theorem \ref{theorem:HS}}\label{subsection:proof_theorem 1} Since any linear map $\bbH^n \rightarrow \bbH^m$ can be trivially extended to one of the form $\bbH^\ell \rightarrow \bbH^\ell$ ($\ell=m$ or $\ell=n$), either by enlarging the source space or the target space, we will assume without loss of generality that the maps in $\cT$ are of the form $\bbH^n \rightarrow \bbH^n$.
	
	It suffices that for arbitrary $\tau_1,\tau_2\in\cT$ we exhibit a non-empty Zariski open subset of $\Gr_{\bbH}(d,n)$ on which every subspace $\cV$ satisfies $\hsp(\cV,\{\tau_1,\tau_2\})$. This will imply $\hsp(\cV,\cT)$ since $\cT$ is a finite set and the intersection of finitely many non-empty Zariski open subsets of $\Gr_{\bbH}(d,n)$ is also non-empty and open. We divide the proof in two cases, $\dim(\cE_{(\tau_1,\tau_2),1})\leq n-d$ and  $\dim(\cE_{(\tau_1,\tau_2),1})> n-d$. Assume that we are in the first case. Then we have the following proposition whose proof is placed at \S \ref{subsubsection:proof_prop:E<=n-d}.
\begin{prop}\label{prop:E<=n-d}
In addition to the hypotheses of Theorem \ref{theorem:HS}, further assume $\dim (\cE_{(\tau_1,\tau_2),1})\leq n-d$. Then there is an $n$-dimensional subspace $\cV^*$ of $\bbH^n$ which satisfies $\dim (\tau_1(\cV^*)+\tau_2(\cV^*))=2d$.
\end{prop}
	With the subspace $\cV^*$ of Proposition \ref{prop:E<=n-d} we get that the set $\mathscr{U}_1$ of subspaces $\cV \in \Gr_{\bbH}(d,n)$ for which $\dim(\tau_1(\cV)+\tau_2(\cV))=2d$ is non-empty. Moreover, $\mathscr{U}_1$ is open. To see this, let $V \in\bbH^{n\times d}$ have $\cV \in \mathscr{U}_1$ as its column space. Then $\dim(\tau_1(\cV)+\tau_2(\cV))=2d$ is equivalent to $\rank [T_1 V \ \, T_2 V]=2d$, which in turn is equivalent to the non-vanishing of some $2n\times 2n$ minor of $[T_1 V \, \, T_2 V]$. Each such minor is a quadratic polynomial in the Pl\"ucker coordinates of $\cV$, so that their non-simultaneous vanishing indeed gives an open set of $\Gr_{\bbH}(d,n)$. We next show that $\hsp(\cV,\{\tau_1,\tau_2\})$ holds for every $\cV\in \mathscr{U}_1$. Indeed, let $v_1,v_2\in\cV$ be such that $\tau_1(v_1)=\tau_2(v_2)$. But $\dim(\tau_1(\cV)+\tau_2(\cV))=2d$ implies  $\tau_1(\cV)\cap\tau_2(\cV)=0$ and also $\dim(\tau_1(\cV))=\dim(\tau_2(\cV))=\dim(\cV)=d$. So $\ker(\tau_1)\cap\cV=0$ and $\ker(\tau_2)\cap\cV=0$. We conclude that $\tau_1(v_1)=\tau_2(v_2)=0$ and moreover $v_1=v_2=0$.
	
	We tackle the second case $\dim(\cE_{(\tau_1,\tau_2),1})> n-d$ by the following proposition, proved in \S \ref{subsubsection:proof-prop:E>n-d}.
	\begin{prop}\label{prop:E>n-d}
		In addition to the hypotheses of Theorem \ref{theorem:HS}, 
		suppose $\dim(\cE_{(\tau_1,\tau_2),1})=n-d_0>n-d$. There are two subspaces $\cV_0^* \subset \cV^*$ of $\bbH^n$ of dimension $d_0$ and $d$ respectively so that $\dim(\tau_1(\cV_0^*)+\tau_2(\cV^*))=d_0+d$.
	\end{prop}	 With $\cV_0^*$ and $\cV^*$ of Proposition \ref{prop:E>n-d} we know that
	\begin{align*}
		\rU_2:= \{ (\cV_0,\cV)\in \F_{\bbH}(d_0,d,n): 
		\dim_\bbH(\tau_1(\cV_0)+\tau_2(\cV))=d_0+d \}
	\end{align*}
	is not empty. By a similar argument that showed $\mathscr{U}_1$ is open, $\mathscr{U}_2$ is also open in $\F_{\bbH}(d_0,d,n)$. Now, Lemma \ref{lem:flag} shows that $\mathscr{U}_2$ induces a non-empty open set $\mathscr{U}_3 \subset \Gr_{\bbH}(d,n)$ such that for every $\cV \in \mathscr{U}_3$ there exists a $\cV_0 \in \Gr_{\bbH}(d_0,n)$ with $(\cV_0,\cV) \in \mathscr{U}_2$. We show that $\hsp(\cV,\{\tau_1,\tau_2\})$ holds for any $\cV \in \mathscr{U}_3$. So suppose that $\tau_1(v_1)=\tau_2(v_2)$ with $v_1,v_2\in\cV$ and let $\cV_0$ be as above. Thus $\dim(\tau_1(\cV_0)+\tau_2(\cV))=d_0+d$ and in particular $\dim(\tau_1(\cV_0)+\tau_2(\cV_0))=2d_0$. Necessarily $\cV_0\cap \cE_{(\tau_1,\tau_2),1}=0$.  
	Now by hypothesis, $\cV \cap \cE_{(\tau_1,\tau_2),1}$ has dimension at least $d+n-d_0-n=d-d_0$. In fact, we must have $\dim(\cV\cap\cE_{(\tau_1,\tau_2),1})=d-d_0$ otherwise the subspaces $\cV_0$ and $\cV\cap\cE_{(\tau_1,\tau_2),1}$ of $\cV$ must intersect, contradicting the fact $\cV_0$ does not intersect $\cE_{(\tau_1,\tau_2),1}$. We conclude that  $\cV_0$ and $\cV \cap \cE_{(\tau_1,\tau_2),1}$ are subspace complements in $\cV$. Write $v_1$ as the sum of two vectors $v_0$ and $u$ in $\cV_0$ and $\cV\cap\cE_{(\tau_1,\tau_2),1}$ respectively. Then $\tau_1(v_1)=\tau_2(v_2)$ implies $\tau_1(v_0)=\tau_2(v_2)-\tau_1(u)=\tau_2(v_2-u)$. Since $v_0\in\cV_0$, $(v_2-u)\in\cV$ and $(\cV_0,\cV)\in\rU_2$, the definition of $\mathscr{U}_2$ implies that $v_0=0$ and $v_2-u=0$. That is, $v_1=u=v_2$. \qed	

	\subsubsection{Proof of Proposition \ref{prop:E<=n-d}}\label{subsubsection:proof_prop:E<=n-d}	
	
	For $\bbH=\bbR$ we have $T_1,T_2\in\bbR^{n\times n}$ and it suffices to show that $\rank[T_1V\, \, T_2 V]=2d$ for some $V\in\bbR^{n\times d}$. This is equivalent to showing some $2d\times 2d$ minor of $[T_1 V\, \, T_2 V]$ is a nonzero polynomial in the entries of $A$. This is certainly true if the evaluation of that minor is non-zero for some $V^*\in\bbC^{n\times d}$. Hence it suffices to prove Proposition \ref{prop:E<=n-d} for $\bbH=\bbC$, which will be the field of choice for the rest of this section. 
	
	We start by introducing a sequence of subspaces and study some useful properties. Set $\cR_0, \cF_0$ equal to $\bbC^n$. For any non-negative integer $j$ define
	\begin{align}
		\cG_{j+1}&=\tau_1(\cR_j\cap\cF_j)\cap \tau_2(\cR_j\cap\cF_j),\nonumber\\
		\cR_{j+1}&=\tau_1^{-1}(\cG_{j+1})\cap \cR_j\cap\cF_j,\label{eq:def-RFG} \\
		\cF_{j+1}&=\tau_2^{-1}(\cG_{j+1})\cap \cR_j\cap\cF_j.\nonumber
	\end{align} Part iv) of the next lemma gives some motivation behind the definition of recursions\footnote{Our recursions bear some resemblance with the somewhat less complicated Wong sequence \cite{Wong-1974}, which for example was used in \cite{Berger-LAA2012Quasi-Weierstrab,Berger-SIAM-J-MAA2012}. A detailed comparison of the two constructions is left as future work.} \eqref{eq:def-RFG}:
	\begin{lemma}\label{lemma:chain-property}
		For any non-negative integer $j$ we have i) $\cR_{j+1}\subset \cR_j\cap\cF_j
		\subset\cR_j$, ii) $\cF_{j+1}\subset \cR_j\cap\cF_j\subset \cF_j$, 
		iii) $\cG_{j+2}\subset\cG_{j+1}$ and iv) $\tau_1(\cR_{j+1})=\tau_2(\cF_{j+1})=\cG_{j+1}$.
	\end{lemma} 
	\begin{proof}
	i) and ii) are directly by definition and so is $\cR_{j+1}\cap\cF_{j+1}\subset \cR_{j}\cap\cF_{j}$. This latter implies
		\begin{align*}			\cG_{j+2}=\tau_1(\cR_{j+1}\cap\cF_{j+1})\cap\tau_2(\cR_{j+1}\cap\cF_{j+1})\subset \tau_1(\cR_j\cap\cF_j)\cap \tau_2(\cR_j\cap\cF_j)=\cG_{j+1}.
		\end{align*}
We now show that $\tau_1(\cR_{j+1})= \cG_{j+1}$. From $\cR_{j+1}\subset \tau_1^{-1}(\cG_{j+1})$ we have $\tau_1(\cR_{j+1})\subset \tau_1(\tau_1^{-1}(\cG_{j+1})) \subset \cG_{j+1}$. For the reverse direction $\cG_{j+1}\subset \tau_1(\cR_{j+1})$ let $z\in\cG_{j+1}= \tau_1(\cR_j\cap\cF_j)\cap \tau_2(\cR_j\cap\cF_j)$. In particular $z\in\tau_1(\cR_j\cap\cF_j)$ so there is some $w\in\cR_j\cap\cF_j$ with $\tau_1(w)=z$. Then $w\in\tau_1^{-1}(z)\cap \cR_j\cap\cF_j$. But $\tau_1^{-1}(z)\subset\tau_1^{-1}(\cG_{j+1})$ and so $w\in\tau_1^{-1}(\cG_{j+1})\cap \cR_j\cap\cF_j=\cR_{j+1}$. Hence $z\in\tau_1(\cR_{j+1})$. A similar derivation gives $\tau_2(\cF_{j+1})=\cG_{j+1}$. \qed
	\end{proof}

Lemma \ref{lemma:chain-property} gives two subspace chains
		$\cdots\subset \cR_{j+1}\subset\cR_{j}\subset\cdots\subset\cR_0$ and $\cdots\subset \cF_{j+1}\subset\cF_{j}\subset\cdots\subset\cF_0$. These stabilize at a common subspace: 
	\begin{lemma}\label{lemma:stabilize}
		There is a non-negative integer $\alpha$ such that $\cR_{j}=\cF_{j}$ for every $j \ge \alpha$ and $\tau_1(\cR_{\alpha})=\tau_2(\cR_{\alpha}) =\cG_\alpha$.
	\end{lemma}
	\begin{proof}
Since the subspaces $\cR_0$ and $\cF_0$ are of finite dimension $m$, both chains stabilize, that is, there exist non-negative integers $\alpha_1$ and $\alpha_2$ such that for any $j_1\geq\alpha_1$ and $j_2\geq\alpha_2$ we have $\cR_{j_1}=\cR_{j_1+1}$ and $\cF_{j_2}=\cF_{j_2+1}$. Let $\alpha:=\max \{\alpha_1,\alpha_2\}$. We then have $\cR_{\alpha}=\cR_{\alpha+1}$ and $\cF_{\alpha}=\cF_{\alpha+1}$. Lemma \ref{lemma:chain-property} and the definition of $\alpha$ give $\cR_{\alpha+1}\subset \cR_\alpha\cap\cF_\alpha\subset \cR_{\alpha}=\cR_{\alpha+1}$. This implies $\cR_{\alpha}=\cR_\alpha\cap\cF_\alpha$. Similarly we get $\cF_{\alpha}=\cR_\alpha\cap\cF_\alpha$. It follows that $\cR_{\alpha}=\cF_{\alpha}$. The equality $\tau_1(\cR_{\alpha})=\tau_2(\cR_{\alpha})= \cG_\alpha$ is immediate from Lemma \ref{lemma:chain-property}. \qed			
	\end{proof}
The strategy of the proof is to make use of the $\cR$ chain ascending from $\cR_\alpha$
	\begin{align}\label{eq:chain}
		\cR_\alpha\subset\cR_{\alpha-1}\cap\cF_{\alpha-1}\subset \cR_{\alpha-1}\subset\cdots\subset \cR_1\subset \cR_0\cap\cF_0\subset \cR_0=\bbC^n
	\end{align} in the following fashion. First we show that there is always a $j_0$, such that either there is a subspace $\cW_{j_0} \subset \cR_{j_0}$ of dimension $[\dim(\cR_{j_0})-(n-d)]$ with $\dim(\tau_1(\cW_{j_0})+\tau_2(\cW_{j_0}))=2\dim(\cW_{j_0})$ or there is a subspace $\cZ_{j_0} \subset \cR_{j_0} \cap \cF_{j_0}$ of dimension $[\dim(\cR_{j_0}\cap\cF_{j_0})-(n-d)]$ with $\dim(\tau_1(\cZ_{j_0})+\tau_2(\cZ_{j_0}))=2\dim(\cZ_{j_0})$. Then we describe devices to pass either from $\cW_{j_0}$ to $\cZ_{j_0-1}$ or from $\cZ_{j_0}$ to $\cW_{j_0-1}$, all the while preserving the properties i) $\dim(\tau_1(\cZ_{{j_0}-1})+\tau_2(\cZ_{{j_0}-1}))=2\dim(\cZ_{{j_0}-1})$ and $\dim(\cZ_{j_0-1}) =[\dim(\cR_{j_0-1}\cap\cF_{j_0-1})-(n-d)]$  or ii) $\dim(\tau_1(\cW_{{j_0}-1})+\tau_2(\cW_{{j_0}-1}))=2\dim(\cW_{{j_0}-1})$ and $\dim(\cW_{j_0-1}) =[\dim(\cR_{j_0-1})-(n-d)]$. Then inductively $\cV^*:= \cW_0$ will satisfy the statement of the proposition since $\cR_0 = \bbC^n$. Below, we distinguish between three cases, in two out of which the existence of a $\cW_{j_0}$ is proved while the third one proves the existence of a $\cZ_{j_0}$. The next lemma handles the case where $j_0$ can be taken to be $\alpha$. 
		
\begin{lemma}[$\cW_\alpha$-Initialization]\label{lemma:W_alpha}
In addition to the hypotheses of Proposition \ref{prop:E<=n-d}, suppose $\dim(\cR_\alpha)>n-d$. Then there is a subspace $\cW_\alpha$ of $\cR_{\alpha}$ of dimension $[\dim(\cR_\alpha)-(n-d)]$ such that $\dim(\tau_1(\cW_\alpha)+\tau_2(\cW_\alpha))=2\dim(\cW_\alpha)$.
	\end{lemma}
	
	\begin{proof}
		By Lemma \ref{lemma:stabilize} and the rank-plus-nullity theorem we have
		\begin{align*}
		\dim(\ker(\tau_1|_{\cR_\alpha}))=\dim(\cR_\alpha)-\dim(\cG_{\alpha})=\dim(\ker(\tau_2|_{\cR_\alpha}))
		\end{align*} Moreover, by hypothesis $\rank(\tau_1|_{\cR_\alpha}) \ge \dim(\cR_\alpha)-(n-2d)>d$. Hence $$(n-d)+\dim(\cG_{\alpha})-\dim(\cR_{\alpha})=\rank(\tau_1|_{\cR_\alpha})-d>0$$ Since $\dim (\cU_{\tau_1,\tau_2}^{\cl}) = \dim (\cU_{\tau_1,\tau_2}) \le n-d$ and $\dim (\cE_{(\tau_1,\tau_2),1}) \le n-d$ as well,  Lemma \ref{lemma:subspaces_intersect_subspace} gives a subspace $\cH$ of $\cR_\alpha$ of dimension $\dim(\cG_{\alpha})$, such that $\cH \cap \ker(\tau_1|_{\cR_\alpha}) = \cH \cap \ker(\tau_2|_{\cR_\alpha})=0$ and 
					\begin{align}
	    	\dim(\cU^{\cl}_{\tau_1,\tau_2}\cap\cH)\le (n-d)+\dim(\cG_\alpha)-\dim(\cR_\alpha)\label{eq:UcapH} \\
			\dim(\cE_{(\tau_1,\tau_2),1}\cap\cH)\le (n-d)+\dim(\cG_\alpha)-\dim(\cR_\alpha)\label{eq:E1capH} 			
		\end{align} Since $\tau_1(\cH)=\cG_\alpha=\tau_2(\cH)$ we have that $\tau_\cH:=(\tau_1|_\cH)^{-1}\tau_2|_\cH$ is an isomorphism of $\cH$. We are going to get our subspace $\cW_\alpha$ by applying Lemma \ref{lemma:lemma5ofHS} with ambient space $\cH$, $\bbC$-linear map $\tau_\cH$ and subspace dimension $[\dim(\cR_\alpha)-(n-d)]$, this number being positive by hypothesis. There are two things we need to check, the first being that $\dim (\cH) \ge 2 [\dim(\cR_\alpha)-(n-d)]$. Indeed, this is true because  
	 \begin{align*}
			\dim(\cH)\geq 2[\dim(\cR_\alpha)-(n-d)]
			&\Leftrightarrow 2n-2d\geq \dim(\cR_\alpha)+\dim(\cR_{\alpha})-\dim(\cG_\alpha)\\
			&\Leftrightarrow 2n-2d\geq \dim(\cR_\alpha)+\dim(\ker(\tau_1|_{\cR_\alpha})) \\
			&\Leftarrow 2n-2d\geq \dim(\cR_\alpha)+\dim(\ker(\tau_1)) \\
			&\Leftrightarrow (\rank(\tau_1)-2d) + (n-\dim(\cR_{\alpha}))\geq 0
		\end{align*} and the last inequality is true by hypothesis. The second thing that we need to check is that for any $\lambda\in\bbC$ 
		\begin{align*}
			\dim(\cE_{\tau_\cH,\lambda})\leq\dim(\cH)-[\dim(\cR_{\alpha})-(n-d)]
		\end{align*}
When $\lambda=0$, we have $\cE_{\tau_\cH,0} = \ker(\tau_\cH)=0$, because $\tau_\cH$ is an isomorphism. When $\lambda\neq 0$, 
$v \in \cE_{\tau_\cH,\lambda}$ is equivalent to $\tau_1|_\cH(v)=\lambda\tau_2|_\cH(v)$ or equivalently $v\in(\cU_{\tau_1,\tau_2}\cap\cH) \cup (\cE_{(\tau_1,\tau_2),1} \cap \cH)$. This shows that $\cE_{\tau_\cH,\lambda}$ lives in $(\cU_{\tau_1,\tau_2}^{\cl}\cap\cH) \cup (\cE_{(\tau_1,\tau_2),1} \cap \cH)$ and as per \eqref{eq:UcapH} and \eqref{eq:E1capH} that latter set has dimension at most $\dim (\cH) -[\dim(\cR_\alpha)-(n-d)]$. 

Now Lemma \ref{lemma:lemma5ofHS} gives a subspace $\cW_\alpha$ of $\cH$ of dimension $[\dim(\cR_{\alpha})-(n-d)]$ such that $\dim(\cW_\alpha+\tau_\cH(\cW_\alpha))=2\dim(\cW_\alpha)$. Since $\tau_1|_\cH$ is an isomorphism from $\cH$ to $\cG_{\alpha}$ and $\cW_\alpha+\tau_\cH(\cW_\alpha)$ is a subspace of $\cH$, we have that $\cW_\alpha+\tau_\cH(\cW_\alpha) \cong \tau_1|_\cH(\cW_\alpha+\tau_\cH(\cW_\alpha)) = \tau_1(\cW_\alpha) + \tau_2(\cW_\alpha)$. That is, $\dim (\tau_1(\cW_\alpha) + \tau_2(\cW_\alpha)) = 2 \dim (\cW_\alpha)$. \qed
	\end{proof}

If $\alpha = 0$, then Proposition \ref{prop:E<=n-d} is proved by Lemma \ref{lemma:W_alpha}, so we assume $\alpha>0$ for the sequel. If $\dim(\cR_\alpha) \le n-d$, and since $\dim (\cR_0) = n$, then necessarily one of the two following cases must occur in the chain \eqref{eq:chain}. Either there is a $\beta$ with $\dim(\cR_{\beta}\cap\cF_{\beta})\leq n-d<\dim(\cR_{\beta})$ or there is a $\gamma$ with $\dim(\cR_{\gamma+1})\leq n-d<\dim(\cR_{\gamma}\cap\cF_\gamma)$. The next two lemmas show how to choose $\cW_{\beta} \subset \cR_\beta$ and $\cZ_\gamma \subset \cR_\gamma \cap \cF_\gamma$, respectively.  

\begin{lemma}[$\cW_\beta$-Initialization]\label{lemma:RjcapFj->Rj}
In addition to the hypotheses of Proposition \ref{prop:E<=n-d}, suppose $\dim(\cR_{\beta}\cap\cF_{\beta})\leq n-d<\dim(\cR_{\beta})$ for some non-negative integer $\beta$. Then there exists a subspace $\cW_\beta$ of $\cR_{\beta}$ of dimension $[\dim(\cR_{\beta})-(n-d)]$ such that $\dim(\tau_1(\cW_\beta)+\tau_2(\cW_\beta))=2\dim(\cW_\beta)$.
\end{lemma}
	
\begin{proof}		
We have
		\begin{align*}
			[\dim(\cR_{\beta})-(n-d)]+\dim(\ker(\tau_1))=\dim(\cR_\beta)+(d-\rank(\tau_1))\le \dim(\cR_\beta)
		\end{align*}
		and a similar inequality for $\tau_2$. Moreover, 
		\begin{align*}
			[\dim(\cR_{\beta})-(n-d)]+ \dim(\cR_\beta\cap\cF_\beta)\leq \dim(\cR_{\beta})+ [\dim(\cR_{\beta}\cap\cF_{\beta})-(n-d)]\leq \dim(\cR_{\beta}).
		\end{align*}
		Consequently, by Lemma \ref{lemma:subspaces_intersect_subspace} there exists a subspace $\cW_\beta$ of $\cR_\beta$ of dimension $[\dim(\cR_{\beta})-(n-d)]$ which does not intersect $\ker(\tau_1), \ker(\tau_2)$ and $\cR_\beta\cap\cF_\beta$. Clearly $\beta >0$ and Lemma \ref{lemma:chain-property} gives $\tau_1(\cW_\beta)\subset\tau_1(\cR_\beta)=\cG_\beta$. Recalling definition  \eqref{eq:def-RFG}, we have
			\begin{align*}
			\cW_\beta\cap \tau_2^{-1}(\tau_1(\cW_\beta))&\subset \cW_\beta\cap \tau_2^{-1}(\cG_\beta)\\
			&=\cW_\beta\cap \tau_2^{-1}(\cG_\beta)\cap\cR_\beta\\
			&=\cW_\beta\cap \tau_2^{-1}(\cG_\beta)\cap\tau_1^{-1}(\cG_\beta)\cap\cR_{\beta-1}\cap\cF_{\beta-1}\\
			&=\cW_\beta\cap \tau_1^{-1}(\cG_\beta)\cap\cF_\beta\\
			&\subset\cW_\beta\cap\cF_\beta=\cW_\beta\cap\cR_\beta\cap\cF_\beta=0.
		\end{align*}
		In short $\cW_\beta\cap \tau_2^{-1}(\tau_1(\cW_\beta))=0$, and it follows that $\tau_2(\cW_\beta)\cap\tau_1(\cW_\beta)=0$. Recalling that $\cW_\beta\cap\ker(\tau_1)=0$ and $\cW_\beta\cap\ker(\tau_2)=0$, we conclude that $\dim(\tau_1(\cW_\beta)+\tau_2(\cW_\beta))=2\dim(\cW_\beta)$. \qed
	\end{proof}

\begin{lemma}[$\cZ_\gamma$-Initialization]\label{lemma:Rj+1->R_j_cap_F_j}
In addition to the hypotheses of Proposition \ref{prop:E<=n-d}, suppose that $\dim(\cR_{\gamma+1})\leq n-d<\dim(\cR_{\gamma}\cap\cF_\gamma)$ for some non-negative integer $\gamma$. Then there exists a subspace $\cZ_\gamma$ of $\cR_{\gamma}\cap\cF_{\gamma}$ of dimension $[\dim(\cR_{\gamma}\cap\cF_{\gamma})-(n-d)]$, such that $\dim(\tau_1(\cZ_\gamma)+\tau_2(\cZ_\gamma))=2\dim(\cZ_\gamma)$.
	\end{lemma}
\begin{proof}
We have 
\begin{align*}
[\dim(\cR_{\gamma}\cap\cF_{\gamma})-(n-d)]+\dim(\ker(\tau_1))=\dim(\cR_{\gamma}\cap\cF_{\gamma})+(d-\rank(\tau_1))\le\dim(\cR_{\gamma}\cap\cF_{\gamma})
\end{align*} and a similar inequality for $\tau_2$. Moreover, 
		\begin{align*}
	    	[\dim(\cR_{\gamma}\cap\cF_{\gamma})-(n-d)]+\dim(\cR_{\gamma+1})=\dim(\cR_{\gamma}\cap\cF_{\gamma})+[\dim(\cR_{\gamma+1})- (n-d)]\leq \dim(\cR_{\gamma}\cap\cF_{\gamma}).
		\end{align*} Thus Lemma \ref{lemma:subspaces_intersect_subspace} implies the existence of a subspace $\cZ_\gamma$ of $\cR_{\gamma}\cap\cF_{\gamma}$ of dimension $[\dim(\cR_{\gamma}\cap\cF_{\gamma})-(n-d)]$ which does not intersect $\ker(\tau_1), \ker(\tau_2)$ and $\cR_{\gamma+1}$. This gives $\dim(\tau_1(\cZ_\gamma))=\dim(\tau_2(\cZ_\gamma))=\dim(\cZ_\gamma)$. It now suffices to prove $\tau_1(\cZ_\gamma)\cap\tau_2(\cZ_\gamma)=0$. Let $\tau_1(v_1)=\tau_2(v_2)$ for some $v_1,v_2\in \cZ_\gamma$. Then		\begin{align*}
			\tau_1(v_1)\in \tau_1(\cR_\gamma\cap\cF_\gamma)\cap\tau_2(\cR_\gamma\cap\cF_\gamma)=:\cG_{\gamma+1}.
		\end{align*}
		This implies $v_1\in\tau_1^{-1}(\cG_{\gamma+1})$ and so
		\begin{align*}
			v_1\in \cZ_\gamma\cap\tau_1^{-1}(\cG_{\gamma+1})=\cZ_\gamma\cap\tau_1^{-1}(\cG_{\gamma+1})\cap \cR_\gamma\cap\cF_\gamma=\cZ_\gamma \cap \cR_{\gamma+1}=0.
		\end{align*} Thus $v_1=0$ and we have proved $\tau_1(\cZ_\gamma)\cap\tau_2(\cZ_\gamma)=0$.  \qed
	\end{proof}

	\begin{table}
		\caption{Three different types of initialization.}
		\label{table:3cases}
		
		\centering
		\begin{tabular}{lll}
			\toprule
			$\cW_\alpha: \, \, n-d<\dim(\cR_{\alpha})$ & Lemma \ref{lemma:W_alpha} \\
			$\cW_\beta: \, \, \dim(\cR_{\beta}\cap\cF_{\beta})\leq n-d<\dim(\cR_{\beta})$ & Lemma \ref{lemma:RjcapFj->Rj} \\
			$\cZ_\gamma: \, \, \dim(\cR_{\gamma+1})\leq n-d<\dim(\cR_{\gamma}\cap\cF_\gamma)$ & Lemma \ref{lemma:Rj+1->R_j_cap_F_j} \\			
			\bottomrule
		\end{tabular}
	\end{table}
Table \ref{table:3cases} summarizes the three different types of initialization, two giving $\cW_\alpha, \cW_\beta$ and the third one $\cZ_\gamma$. With $\mu$ either $\alpha$ or $\beta$, we have that subspace $\cW_\mu$ of $\cR_\mu$ satisfies
	\begin{align*}
		\rP(\cW_\mu):\ \dim(\cW_{\mu})=[\dim(\cR_{\mu})-(n-d)] \text{\ \ and \ } \dim(\tau_1(\cW_{\mu})+\tau_2(\cW_{\mu}))=2\dim(\cW_{\mu}).
	\end{align*} On the other hand, $\cZ_\gamma$ is a subspace of $\cR_\gamma\cap\cF_\gamma$ and satisfies
	\begin{align*}
		\rP(\cZ_\gamma): \	\dim(\cZ_\gamma)=[\dim(\cR_{\gamma}\cap\cF_{\gamma})-(n-d)] \text{\ \ and \ } \dim(\tau_1(\cZ_\gamma)+\tau_2(\cZ_\gamma))=2\dim(\cZ_\gamma).
	\end{align*} Thus, either we have a chain of the form 
	\begin{align}
			\cW_\mu \subset \cR_{\mu}\subset \cR_{\mu-1}\cap\cF_{\mu-1}\subset  \cdots\subset\cR_0=\bbC^n \nonumber 	
	\end{align} or a chain of the form			 
	\begin{align}
		\cZ_\gamma \subset \cR_\gamma&\cap\cF_{\gamma}\subset\cR_{\gamma}\subset  \cdots\subset\cR_0=\bbC^n. \nonumber 		
	\end{align} The next two lemmas show that we can always extend $\cW_\mu$ to $\cZ_{\mu-1}$ or $\cZ_\gamma$ to $\cW_{\gamma-1}$ and so on, which enables induction and thus concludes the proof of the proposition. The proof of Lemma \ref{lemma:enlarge_R_W->RcapF_Z} follows an identical argument as in the proof of Lemma \ref{lemma:enlarge_RcapF_Z->R_W} and is thus omitted.
	
	\begin{lemma}[$\cW_j$-Extension]\label{lemma:enlarge_RcapF_Z->R_W}
		In addition to the hypotheses of Proposition \ref{prop:E<=n-d}, suppose for some $j$ that $\dim(\cR_{j}\cap\cF_{j})> n-d$ and that there exists a subspace $\cZ_j$ of $\cR_j\cap\cF_j$ satisfying
		\begin{align*}
			\rP(\cZ_j):\ \dim(\cZ_j)=[\dim(\cR_{j}\cap\cF_{j})-(n-d)] \text{\ \ and \ } \dim(\tau_1(\cZ_j)+\tau_2(\cZ_j))=2\dim(\cZ_j).
		\end{align*}
		Then there exists a subspace $\cW_j$ of $\cR_{j}$ satisfying $\cZ_j\subset \cW_j$ and
		\begin{align*}
			\rP(\cW_j):\ \dim(\cW_j)=[\dim(\cR_{j})-(n-d)] \text{\ \ and \ } \dim(\tau_1(\cW_j)+\tau_2(\cW_j))=2\dim(\cW_j).
		\end{align*}
	\end{lemma}
	\begin{proof}
		If $\cR_j\cap\cF_j=\cR_j$, then we are done by letting $\cW_j=\cZ_j$. In what follows we assume $\dim(\cR_j)>\dim(\cR_j\cap\cF_j)$, in particular $j > 0$. The subspace $\tau_1^{-1}(\tau_1(\cZ_j)+\tau_2(\cZ_j))$ has dimension at most $(n-\rank(\tau_1))+2[\dim(\cR_j\cap\cF_j)-(n-d)]$. Hence 
			\begin{align*}
			[\dim(\cR_j)-\dim(\cR_j\cap\cF_j)]+\dim(\tau_1^{-1}(\tau_1(\cZ_j)+\tau_2(\cZ_j)))&\le \dim(\cR_j)+[2d-\rank(\tau_1)]+[\dim(\cR_j\cap\cF_j)-n] \nonumber \\
			&\leq \dim(\cR_j).
		\end{align*} By hypothesis it is also true that $[\dim(\cR_j)-\dim(\cR_j\cap\cF_j)]+\dim(\ker(\tau_1)) \le \dim(\cR_j)$ and similarly for $\tau_2$. Hence, by Lemma \ref{lemma:subspaces_intersect_subspace} there is a subspace $\cW_j'$ of $\cR_j$ of dimension $[\dim(\cR_j)-\dim(\cR_j\cap\cF_j)]$, which does not intersect the subspaces $\cR_j\cap\cF_j, \, \tau_1^{-1}(\tau_1(\cZ_j)+\tau_2(\cZ_j)), \, \ker(\tau_1), \, \ker(\tau_2)$. In particular $\tau_1(\cW_j')\cap[\tau_1(\cZ_j)+\tau_2(\cZ_j)]=0$. This together with the hypothesis gives $\dim (\tau_1(\cZ_j)+\tau_2(\cZ_j)+\tau_1(\cW_j')) = 2[\dim(\cR_{j}\cap\cF_{j})-(n-d)]+[\dim(\cR_j)-\dim(\cR_j\cap\cF_j)]$. Equivalently,
		\begin{align}\label{eq:Zj+Wj'_Zj}
			\dim(\tau_1(\cZ_j+\cW_j')+\tau_2(\cZ_j))=\dim(\cR_{j})+\dim(\cR_{j}\cap\cF_{j})-2(n-d).
		\end{align}
		Since $\cZ_j\subset \cR_j\cap\cF_j$ we see that $\tau_2(\cZ_j)\subset \tau_2(\cF_j)=\cG_j$. With $\tau_1(\cZ_j+\cW_j')\subset \tau_1(\cR_j)=\cG_j$ (Lemma \ref{lemma:chain-property}), we obtain that $\tau_1(\cZ_j+\cW_j')+\tau_2(\cZ_j)$ is a subspace of $\cG_j$, and consequently
		\begin{align*}
			\cW_j'\cap\tau_2^{-1}(\tau_1(\cZ_j+\cW_j')+\tau_2(\cZ_j))&\subset \cW_j'\cap\tau_2^{-1}(\cG_j)\\
			&=\cW_j'\cap\tau_2^{-1}(\cG_j) \cap\cR_j\\
			&=\cW_j'\cap\tau_2^{-1}(\cG_j) \cap\tau_1^{-1}(\cG_j) \cap\cR_{j-1}\cap\cF_{j-1}\\
			&=\cW_j'\cap \tau_1^{-1}(\cG_j)\cap\cF_j\\
			&\subset \cW_j'\cap\cF_j\\
			&=\cW_j'\cap\cF_j\cap\cR_j=0. 
		\end{align*}
In short, we have $\cW_j'\cap\tau_2^{-1}(\tau_1(\cZ_j+\cW_j')+\tau_2(\cZ_j))=0$ and so $\tau_2(\cW_j')\cap[\tau_1(\cZ_j+\cW_j')+\tau_2(\cZ_j)]=0$. Recalling \eqref{eq:Zj+Wj'_Zj}, it follows that $[\tau_1(\cZ_j+\cW_j')+\tau_2(\cZ_j)]+\tau_2(\cW_j')$ is of dimension $[\dim(\cR_{j})+\dim(\cR_{j}\cap\cF_{j})-2(n-d)]+ [\dim(\cR_j)-\dim(\cR_j\cap\cF_j)]$, that is,
		\begin{align*}
			\dim(\tau_1(\cZ_j+\cW_j')+\tau_2(\cZ_j+\cW_j'))=2\dim(\cR_{j})-2(n-d).
		\end{align*}
		By letting $\cW_j=\cZ_j+\cW_j'$ we finished the proof. \qed
	\end{proof}
	
	\begin{lemma}[$\cZ_j$-Extension]\label{lemma:enlarge_R_W->RcapF_Z}
		In addition to the hypotheses of Proposition \ref{prop:E<=n-d}, suppose for some $j$ that $\dim(\cR_{j+1})> n-d$ and that there exists a subspace $\cW_{j+1}$ of $\cR_{j+1}$ satisfying 
		\begin{align*}
			\rP(\cW_{j+1}):\ \dim(\cW_{j+1})=[\dim(\cR_{j+1})-(n-d)] \text{\ \ and \ } \dim(\tau_1(\cW_{j+1})+\tau_2(\cW_{j+1}))=2\dim(\cW_{j+1}).
		\end{align*}
		Then there exists a subspace $\cZ_j$ of $\cR_{j}\cap\cF_{j}$ satisfying $\cW_{j+1}\subset \cZ_j$ and 
		\begin{align*}
			\rP(\cZ_j):\ \dim(\cZ_j)=[\dim(\cR_{j}\cap\cF_{j})-(n-d)] \text{\ \ and \ } \dim(\tau_1(\cZ_j)+\tau_2(\cZ_j))=2\dim(\cZ_j).
		\end{align*} \qed
	\end{lemma}

\subsubsection{Proof of Proposition \ref{prop:E>n-d}}\label{subsubsection:proof-prop:E>n-d}
		Note that $\rank(\tau_1),\rank(\tau_2) \geq 2d>2d_0$ and $\dim(\cU_{\tau_1,\tau_2})\leq n-d<n-d_0$. 
		Invoking Proposition \ref{prop:E<=n-d}, we get a subspace $\cV_0$ of 
		$\Gr_{\bbH}(d_0,n)$ which satisfies 
		$\dim(\tau_1(\cV_0)+\tau_2(\cV_0))=2d_0$. The dimension of the subspace $\tau_2^{-1}(\tau_1(\cV_0)+\tau_2(\cV_0))$ is at most $(n-\rank(\tau_2))+2d_0$, and
		\begin{align*}
			(d-d_0) +[(n-\rank(\tau_2))+2d_0]=n+(d+d_0-\rank(\tau_2))<n+2d-\rank(\tau_2)\leq n.
		\end{align*} By Lemma \ref{lemma:subspaces_intersect_subspace}, there is a subspace $\cW$ of $\bbH^n$ of dimension $d-d_0$ such that $\cW$ does not intersect the subspaces $\tau_2^{-1}(\tau_1(\cV_0)+\tau_2(\cV_0)), \, \cV_0$ and $\ker(\tau_2)$. Hence  $\dim(\cW+\cV_0)=d, \, \tau_2(\cW) \cap (\tau_1(\cV_0)+\tau_2(\cV_0))=0$ and
		\begin{align*}
			\dim(\tau_1(\cV_0)+\tau_2(\cW+\cV_0))=\dim(\tau_2(\cW))
			+ \dim 
			(\tau_1(\cV_0)+\tau_2(\cV_0))=d-d_0+2d_0=d+d_0.
		\end{align*}
		Letting 
		$\cV=\cW+\cV_0$ we are done. \qed

\subsection{Proof of Proposition \ref{prop:HS}} 
	Any $\cV \in\Gr_{\bbC}(d,n)$ that intersects $\cU_{\tau_1,\tau_2}$ violates $\hsp(\cV,\cT)$. So it suffices to show $\cV\cap\cU_{\tau_1,\tau_2}$ is not empty for a generic $\cV\in\Gr_{\bbC}(d,n)$. This follows from Lemma \ref{lemma:homogeneous}, proved in \S \ref{proof:lemma:homogeneous}, and the 
fact that $\cU_{\tau_1,\tau_2} = \cY_{\tau_1,\tau_2} \setminus \cZ_{\tau_1,\tau_2}$, with both $\cY_{\tau_1,\tau_2}$ and $\cZ_{\tau_1,\tau_2}$ defined by homogeneous polynomials.  \qed

\subsection{Proof of Theorem \ref{theorem:HS-SA}} \label{subsection:proof_theorem2}
	
Set $\sG = \Gr_\bbH(d_1,n) \times \cdots \times \Gr_\bbH(d_\ell,n)$ and for every $\cI \subset [\ell]$ denote by $\sG_\cI$ the product of the factors of $\sG$ indexed by $\cI$. It is clear that an open set of $\sG_\cI$ gives rise to an open set of $\sG$, with the closed locus in $\sG$ to be avoided defined by equations involving only the Pl\"ucker coordinates of the factors indexed by $\cI$. To show that $\hsp(\ocAI,\cT)$ holds true for every subspace arrangement $(\cV_1,\dots,\cV_\ell)$ on a non-empty open set $\sU$ of $\sG$, it suffices to show that $\hsp(\cV_{\cI_i}\cup \cV_{\cI_j},\{\tau_\alpha,\tau_\beta\})$ holds true on a non-empty open set $\sU_{i,j,\alpha,\beta}$ of $\sG_{\cI_i \cup \cI_j}$ for every $\cI_i, \cI_j \in \rI$ and for every $\tau_\alpha,\tau_\beta \in \cT$. For then $\sU$ will be the intersection of all $\sU_{i,j,\alpha,\beta}$'s, viewed as open sets of $\sG$. With $i,j,\alpha,\beta$ fixed, we show the existence of such a $\sU_{i,j,\alpha,\beta}$.

The dimension of the subspace $\tau_\alpha(\cV_{\cI_i}) + \tau_\beta(\cV_{\cI_j})$ attains its maximum possible value, say $c$, on a non-empty open set $\sU_{\cI_i\cup \cI_j}$ of $\sG_{\cI_i \cup \cI_j}$. To see this, let $V_k$ be an $m \times n_k$ matrix with a basis of $\cV_k$ in its columns. Let $V_{\cI_i}$ be the column-wise concatenation of those $V_k$'s with $k \in \cI_i$. Define similarly $V_{\cI_i}$ and $V_{\cI_i \cup \cI_j}$. Let us view the entries of the $V_k$'s as polynomial variables and consider the polynomial ring $\bbH[V_{\cI_i \cup \cI_j}]$ whose elements are polynomials in the variables $V_{\cI_i \cup \cI_j}$ and coefficients in $\bbH$. Let $\bbH(V_{\cI_i \cup \cI_j})$ be the field of fractions of $\bbH[V_{\cI_i \cup \cI_j}]$, that is every element of $\bbH(V_{\cI_i \cup \cI_j})$ is of the form $f/g$ with $f,g \in \bbH[V_{\cI_i \cup \cI_j}]$ and $g \neq 0$. Then the matrix $[T_\alpha V_{\cI_i} \, \, \, T_\beta V_{\cI_i \cap \cI_j}]$ is an element of $\bbH(V_{\cI_i \cup \cI_j})^{m \times (n_{\cI_i}+n_{\cI_j})}$ and $c$ coincides with its rank over $\bbH(V_{\cI_i \cup \cI_j})$. Moreover, $\sU_{\cI_i\cup \cI_j}$ is defined by the non-simultaneous vanishing of all $c \times c$ determinants of that matrix, which are polynomials in the Pl\"ucker coordinates of the $V_k$'s.  

We claim that for every subspace arrangement $(\cV_k)_{k \in \cI_i \cup \cI_j} \in \sU_{\cI_i\cup \cI_j}$ the subspace $\tau_{\beta}(\cV_{\cI_j \setminus \cI_i})$ does not intersect $\tau_\alpha(\cV_{\cI_i}) + \tau_\beta(\cV_{\cI_i \cap \cI_j})$. To see this, note $\tau_\alpha(\cV_{\cI_i}) + \tau_\beta(\cV_{\cI_j}) = \tau_\alpha(\cV_{\cI_i}) + \tau_\beta(\cV_{\cI_i \cap \cI_j})+ \tau_{\beta}(\cV_{\cI_j \setminus \cI_i})$ and 
\begin{align}
c=\dim(\tau_\alpha(\cV_{\cI_i}) + \tau_\beta(\cV_{\cI_i \cap \cI_j})+ \tau_{\beta}(\cV_{\cI_j \setminus \cI_i}))=&\dim(\tau_\alpha(\cV_{\cI_i}) + \tau_\beta(\cV_{\cI_i \cap \cI_j}))+\dim(\tau_{\beta}(\cV_{\cI_j \setminus \cI_i})) \nonumber \\
&-\dim((\tau_\alpha(\cV_{\cI_i}) + \tau_\beta(\cV_{\cI_i \cap \cI_j}))\cap  \tau_{\beta}(\cV_{\cI_j \setminus \cI_i}))  \nonumber
\end{align} By hypothesis $2d_{\cI_i}\le \rank(\tau_\beta)$ and  $2d_{\cI_j} \le \rank(\tau_\beta)$, thus
$$\dim(\tau_\alpha(\cV_{\cI_i}) + \tau_\beta(\cV_{\cI_i\cap \cI_j}))\le d_{\cI_i}+d_{\cI_i \cap \cI_j} \le \rank(\tau_\beta)- d_{\cI_j \setminus \cI_i}$$ Now, if $\tau_{\beta}(\cV_{\cI_j \setminus \cI_i})$ intersects $\tau_\alpha(\cV_{\cI_i}) + \tau_\beta(\cV_{\cI_i \cap \cI_j})$, there is another arrangement obtained by setting $\cV_k' = \cV_k$ for every $k \in \cI_i$ and replacing the $\cV_k$'s with $k \in \cI_j \setminus \cI_i$ by suitable $\cV_k', \, k\in \cI_j \setminus \cI_i$, such that i) $\dim \tau_{\beta}(\cV'_{\cI_j \setminus \cI_i}) = n_{\cI_j \setminus \cI_i}$ and ii) $\tau_{\beta}(\cV'_{\cI_j \setminus \cI_i})$ does not intersect $\tau_\alpha(\cV_{\cI_i}) + \tau_\beta(\cV_{\cI_i \cap \cI_j})$. Such a replacement is always possible. But then $\dim(\tau_\alpha(\cV_{\cI_i}') + \tau_\beta(\cV_{\cI_j}'))> \dim(\tau_\alpha(\cV_{\cI_i}) + \tau_\beta(\cV_{\cI_j}))$, a contradiction on the maximality of $c$. A similar argument shows that the same property is true if we interchange the roles of $i$ and $j$. In the sequel, we will obtain $\sU_{i,j,\alpha,\beta}$ by intersecting $\sU_{\cI_i\cup \cI_j}$ with several other suitable non-empty open sets. 

By dimension considerations, there is a non-empty open set $\sU_{\cI_i\cup \cI_j}'$ of $\sG_{\cI_i \cup \cI_j}$ such that the $\cV_k$'s are independent subspaces for every subspace arrangement in $\sU_{\cI_i\cup \cI_j}'$, that is $\dim (\cV_{\cI_i\cup \cI_j}) = d_{\cI_i\cup \cI_j} = \sum_{k \in \cI_i \cup \cI_j} d_k$. Hence we have a surjective map $\varphi_i:\sU_{\cI_i\cup \cI_j}' \rightarrow \Gr_{\bbH}(d_{\cI_i},n)$, which sends $(\cV_k)_{k \in \cI_i \cup \cI_j}$ to $\cV_{\cI_i}$. By Theorem \ref{theorem:HS} there is a non-empty open set $\sU_{\cI_i}$ of $\Gr_{\bbH}(d_{\cI_i},n)$ such that $\hsp(\cV,\{\tau_\alpha,\tau_\beta\})$ holds true for every $\cV \in \sU_{\cI_i}$. Similarly, there is a non-empty open set $\sU_{\cI_j}$ of $\Gr_{\bbH}(d_{\cI_j},n)$ such that $\hsp(\cV,\{\tau_\alpha,\tau_\beta\})$ holds true for every $\cV \in \sU_{\cI_j}$. Now, intersect $\sU_{\cI_i\cup \cI_j}$ with $f_i^{-1}(\sU_{\cI_i}) \cap f_j^{-1}(\sU_{\cI_j}) \cap \sU_{\cI_i\cup \cI_j}''$ and call the result again $\sU_{\cI_i\cup \cI_j}$, here $\sU_{\cI_i\cup \cI_j}''$ is the open set where $\cV_{\cI_i}, \cV_{\cI_j}$ do not intersect $\ker(\tau_\alpha), \ker(\tau_\beta)$. 

We now show that $\sU_{\cI_i\cup \cI_j}$ is the required $\sU_{i,j,\alpha,\beta}$. Note that, by the definition of $\sU_{\cI_i\cup \cI_j}$, we only need to consider the case $\tau_\alpha(v_i) = \tau_\beta(v_j)$ with $v_i \in \cV_{\cI_i}$ and $v_j \in \cV_{\cI_j} $. Write $v_j = v_{j\setminus i} + v_{i \cap j}$ where $v_{j\setminus i}  \in \cV_{\cI_j \setminus \cI_i}$ and $v_{i \cap j} \in \cV_{\cI_j \cap \cI_i}$. We have $\tau_\beta(v_{j \setminus i}) = \tau_\alpha(v_i) - \tau_\beta(v_{j \cap i})$. That is, $\tau_\beta(v_{j \setminus i})$ is in the intersection of $\tau_{\beta}(\cV_{\cI_j \setminus \cI_i})$ with $\tau_\alpha(\cV_{\cI_i}) + \tau_\beta(\cV_{\cI_i \cap \cI_j})$. By what we have said above, $\tau_\beta(v_{j \setminus i})=0$. Thus $v_{j \setminus i} \in \ker(\tau_\beta)$ and by the definition of $\sU_{\cI_i\cup \cI_j}$ we further have $v_{j \setminus i}=0$. Hence $v_j \in \cV_{\cI_i}$ and the equation $\tau_\alpha(v_i) = \tau_\beta(v_j)$ implies $v_i = v_j$ by the definition of $\sU_{\cI_i\cup \cI_j}$. \qed

\subsection{Proof of Theorem \ref{theorem:HS-deterministic-noise}}
		We first rewrite \eqref{eq:HS-MLE} into the following convenient form.
		\begin{align*}
			\hat{\tau}&=\argmin_{\tau\in\cT} \min_{v\in \cV}\norm{\overline{y}-\tau (v)}{2}\\
			&=\argmin_{\tau\in\cT}\min_{w\in \tau(\cV)} \norm{\overline{y}-w}{2}^2\\
			&=\argmin_{\tau\in\cT} \min_{w\in \tau(\cV)}\{\norm{w}{2}^2-\inner{\oy}{w}-\inner{w}{\oy} \}\\
			&=\argmin_{\tau\in\cT} \min_{\lambda>0} \min_{w\in \tau(\cV):\norm{w}=\lambda} \{\lambda^2- \inner{\oy}{w}-\inner{w}{\oy} \}\\
			&=\argmin_{\tau\in\cT} \min_{\lambda>0} \{ \lambda^2 - 2\lambda\norm{\overline{y}} \max_{w\in \tau(\cV):\norm{w}=\lambda}  \frac{\inner{\oy}{w}+\inner{w}{\oy}}{2\norm{\overline{y}}\norm{w}} \}\\
			&=\argmin_{\tau\in\cT} \min_{\lambda>0} \{ \lambda^2 - 2\lambda\norm{\overline{y}} \cos(\overline{y},\tau(\cV)) \}\\
			&= \argmax_{\tau\in\cT}  \cos(\overline{y},\tau(\cV)).
		\end{align*}
	We then prove $\htau\in\cT_1$.
	It suffices to show for any $\tau_2\in \cT\setminus \cT_1$ that there is some $\tau_1\in \cT_1$ so that
	\begin{align*}
		\cos(\overline{y},\tau_1(\cV)) > \cos(\overline{y},\tau_2(\cV)),
	\end{align*}
	which surely holds, if the following stronger condition
	\begin{align}\label{eq:tau_stronger_condition}
		\frac{\inner{\oy}{y}+\inner{y}{\oy}}{2\norm{\oy}{2}\norm{y}{2}}> \cos(\oy,\tau_2(\cV))
	\end{align}
	is satisfied. Letting $w_2\in\tau_2(\cV)$ with $\norm{w_2}=1$ be such that $(\inner{\oy}{w_2}+\inner{w_2}{\oy})/\norm{\overline{y}}=\cos(\overline{y},\tau_2(\cV))$ and recalling that $\overline{y}=y+\epsilon$, condition \eqref{eq:tau_stronger_condition} is equivalent to
	\begin{align*}
		\frac{\inner{\oy}{y}+\inner{y}{\oy}}{\norm{\oy}{2}\norm{y}{2}}> \frac{\inner{\oy}{w_2}+\inner{w_2}{\oy}}{\norm{\overline{y}}}
		\Leftrightarrow & \frac{\inner{\oy}{y}+\inner{y}{\oy}}{\norm{y}{2}^2}> \frac{\inner{\oy}{w_2}+\inner{w_2}{\oy}}{\norm{y}}\\
		\Leftrightarrow & 2 > \frac{\inner{y}{w_2}+\inner{w_2}{y}}{\norm{y}{2}} + \frac{\inner{\epsilon}{w_2}+\inner{w_2}{\epsilon}}{\norm{y}} - \frac{\inner{\epsilon}{y}+\inner{y}{\epsilon}}{\norm{y}{2}^2}.\\
		\Leftarrow& 2>2\cos(y,\tau_2(\cV))+\frac{2\norm{\epsilon}}{\norm{y}} + \frac{2\norm{\epsilon}}{\norm{y}}\\
		\Leftrightarrow& \norm{y}(1-\cos(y,\tau_2(\cV)))>2\norm{\epsilon}
	\end{align*}
	which is already fulfilled by \eqref{eq:deterministic_noise_condition}. Hence $\htau\in\cT_1$. So we have $y=\tau^*(v^*)=\htau (v)$ for some $v\in\cV$. This implies $v=v^*$, and thus $y=\htau(v^*)$.
	On the other hand, according to \eqref{eq:HS-MLE}, we have
	\begin{align*}
		\hv&=\argmin_{v\in\cV}\norm{y+\epsilon-\htau(v)}.
	\end{align*}
	Thus, for $\hx\in\bbH^d$ and $x^*\in\bbH^d$ satisfying $\hv=V\hx$ and $v^*=Vx^*$, we get that 
	\begin{align*}
		\hx&=\argmin_{x\in\bbH^d} \norm{y+\epsilon-\hT Vx}=(\hT V)^\dagger(y+\epsilon),
	\end{align*}
	where we used the fact that $\hT V$ is necessarily of full column rank.
	Recalling $y=\htau(v^*)=\hT V x^*$, we obtain
	\begin{align*}
		\hx =(\hT V)^\dagger(\hT Vx^*+\epsilon)
		=x^*+(\hT V)^\dagger\epsilon,
	\end{align*}
	and consequently $\hv=v^*+V(\hT V)^\dagger\epsilon$. \qed

\subsection{Proof of Theorem \ref{corollary:HS-R^n}}

Part i) is a special case of ii) and we prove the latter. Applying Theorem \ref{theorem:HS} with $n$ set to $m$, $m$ set to $r$, and $d$ set to $n$, and in view of Lemma \ref{lemma:U_dimensionbound}, we get a non-empty open set $\sU$ of $\Gr_\bbR(n,m)$ such that $\hsp(\cV,\cS_{r,m})$ holds true for every $\cV \in \sU$. Now let $\sV$ be the non-empty open set of $\bbR^{m\times n}$ consisting of full-rank matrices. There is a surjective polynomial map $f: \sV \rightarrow \Gr_\bbR(n,m)$, defined in the same way as the Pl\"ucker embedding, which sends $A \in \sV$ to its column space $\operatorname{R}(A)$. Now $f^{-1}(\sU)$ is a non-empty open set of $\bbR^{m\times n}$ such that for every $A \in f^{-1}(\sU)$ we have $\hsp(\operatorname{R}(A),\cS_{r,m})$. Since every $A \in f^{-1}(\sU)$ is of full column rank, we also have $\hsp(\bbR^n, \{A\})$. Parts iii) and iv) follow from  Proposition \ref{prop:HS-sign} and Lemma  \ref{lemma:U_dimensionbound} in a similar fashion.

\subsection{Proof of Theorem \ref{corollary:HS-k-sparse}}

We only prove ii), which implies i). Parts iii) and iv) follow similarly. With $r\geq 2k$, Lemma \ref{lemma:U_dimensionbound} gives $\dim(\cU_{S,S'})\leq m-k$ for any rank-$r$ selections $S,S'\in\cS_{r,m}$. Let $s=\binom{n}{k}$ and let $\rI=(\cI_1,\dots,\cI_s)$ be the set of all subsets of $[n]$ of cardinality $k$, say, ordered in the lexicographic order. Then Theorem \ref{theorem:HS-SA}, applied with $n$ set to $m$, $m$ set to $r$, and $d$ set to $1$, gives a non-empty open set $\sU$ of $\prod_{j\in[n]}\Gr_\bbR(1,m)$, such that for any $\cA=(\cV_1,\dots,\cV_n)\in\sU$, the property $\hsp(\ocAI,\cS_{r,m})$ holds true. Let $\sV$ be the open set of $\bbR^{m\times n}$ on which for every $A \in \sV$ and every $j \in [n]$ the $j$-th column of $A$ is non-zero. We have a surjective map $f: \sV \rightarrow \prod_{j\in[n]}\Gr_{\bbR}(1,m)$ which sends $A=[a_1 \cdots a_n]$ to the subspace arrangement $(\Span(a_1),\dots,\Span(a_n))$. Let $\sV'$ be the set of $A \in \bbR^{m\times n}$ for which any $\min\{n,2k\}$ distinct columns of $A$ are linearly independent. Then $\sV'' = f^{-1}(\sU) \cap \sV'$ is a non-empty open set of $\bbR^{m\times n}$.
	
We show that $\hsp(\ocKI,\cS_{r,m}A)$ holds true for every $A \in \sV''$. Let us view $A$ as the linear map $\tau_A: \bbR^n \rightarrow \bbR^m$ defined by $\tau_A(x)=Ax$. By the definition of $\sV''$ we have that $\hsp(\tau_A(\ocKI),\cS_{r,m})$ holds true. That is, for any $k$-sparse vectors $x, x'\in\ocKI$ and $S,S'\in\cS_{r,m}$ satisfying $SAx=S'Ax'$, we have $Ax=Ax'$. But $A(x-x')=0$ is a linear dependence relation involving at most $\min \{n,2k\}$ columns of $A$ and thus again by the definition of $\sV''$ we must have $x=x'$. \qed

\section{Appendix}\label{section:Appendix}

\subsection{Proof of Lemma \ref{lem:flag}} \label{proof:lem:flag}

We assume familiarity with basic topological considerations in algebraic geometry on the level of schemes, e.g. see \cite{Vakil-AG}. We first treat the case $\bbH = \bbC$, where classical arguments suffice. By Chevalley's theorem \cite{Harris-AG} $\phi(\sU)$ is constructible, that is $\phi(\sU) = \cup_{\nu} \sY_\nu \cap \sU_\nu$ where the $\sY_\nu$'s are closed in $ \Gr_{\bbC}(d,n)$, the $\sU_\nu$'s are open in $ \Gr_{\bbC}(d,n)$, and $\nu$ takes finitely many values. If $\phi(\sU)$ does not contain any non-empty open set, then necessarily it is contained in the proper closed subset $\sY = \cup_{\nu} \sY_\nu$. The complement of $\sY$ is a non-empty open subset of $\Gr_{\bbC}(d,n)$ which does not intersect $\phi(\sU)$, and thus its inverse image under $\phi$ is also a non-empty open subset of $\F_{\bbC}(d_0,d,n)$ not intersecting $\sU$. This implies that $\F_{\bbC}(d_0,d,n)$ can be written as a union of two proper closed sets. This is a contradiction because $\F_{\bbC}(d_0,d,n)$ is irreducible. 

Next, we treat the case $\bbH = \bbR$. Then the arguments in the previous paragraph apply without change providing we treat $\phi$ as a morphism of finite type of Noetherian schemes over $\bbR$; see \cite{Vakil-AG,hartshorne1977algebraic,Eisenbud-2004}. Thus we write $\overline{\phi}: \overline{\F}_{\bbR}(d_0,d,n) \rightarrow \overline{\Gr}_{\bbR}(d,n)$, where the overline indicates the scheme structure. By the Jacobson property, the restriction of $\overline{\phi}$ on the $\bbR$-valued points is just $\phi$. The polynomials that define $\sU$ also define a corresponding scheme $\overline{\sU} \subset \overline{\F}_{\bbR}(d_0,d,n)$, and the above arguments applied to $\overline{\phi}$ show that $\overline{\phi}(\overline{\sU})$ contains a non-empty open subscheme $\overline{\sV}$ of $\overline{\Gr}_{\bbR}(d,n)$. Now $\overline{\Gr}_{\bbR}(d,n)$ is locally isomorphic to the affine space $\bbA^{d(n-d)}=\operatorname{Spec} (\bbR[Z])$, where $Z$ is an $d \times (n-d)$ matrix of variables $z_{ij}$ and $\bbR[Z]$ is the polynomial ring in the $z_{ij}$'s with coefficients over $\bbR$. So let $\overline{\sV'}$ be an open subscheme of $\overline{\Gr}_{\bbR}(d,n)$ isomorphic to $\bbA^{d(n-d)}$. Then $\overline{\sV''} = \overline{\sV'} \cap \overline{\sV}$ is also open in $\overline{\Gr}_{\bbR}(d,n)$ 
and in fact non-empty because $\overline{\Gr}_{\bbR}(d,n)$ is irreducible. Under the isomorphism $\overline{\sV'} \cong \bbA^{d(n-d)}$ we view $\overline{\sV''}$ as a non-empty open subscheme of $\bbA^{d(n-d)}$. Now $\overline{\cV''}$ can be written as $\bigcup_{p} \operatorname{Spec} (k[Z])_p$, with $p \in k[Z]$ and $(k[Z])_p$ the localization of $k[Z]$ at the multiplicatively closed set $\{1,p,p^2,\dots\}$. Since $\overline{\sV''}$ is non-empty, not all $p$'s are zero. Hence there is some non-zero $p$ for which $\overline{\sV'''}=\operatorname{Spec} (k[Z])_p$ is a non-empty open subscheme of $\bbA^{d(n-d)}$. Let $\sU'$ be the open set of points in $\bbR^{d(n-d)}$ which are not roots of $p$. Since $\bbR$ is infinite, $\sU'$ is non-empty. Finally, $\sU'$ lies in the image of $\phi$. \qed

\subsection{Proof of Lemma \ref{lemma:homogeneous}} \label{proof:lemma:homogeneous}

We assume familiarity with basic dimension theory in commutative algebra and algebraic geometry, for example see \cite{Vakil-AG} and \cite{Eisenbud-2004} respectively. Let $\frakR:=\bbC[w_1,\dots,w_n]$ be a polynomial ring associated with $\bbC^n$ and let $J$ and $I$ be the vanishing ideals of $\cZ$ and $\cY$, respectively. Since $\cZ \neq \{0\}$ we have that $J$ is properly contained in the ideal $(w_1,\dots,w_n)$ generated by the $w_i$'s. Let $\cU$ = $\cY\backslash \cZ$. Then the vanishing ideal of the closure $\cU^{\cl}$ of $\cU$ is $\fraka:=I:J^{\infty}$, where $\fraka$ is the saturation of $I$ with respect to $J$. Hence we have $\dim(\frakR/\fraka)=\dim(\cU^{\cl})=\dim(\cU)>n-d$. Since $I,J$ are homogeneous so is $\fraka$. Then for $n-d$ generic linear forms $\ell_1,\dots,\ell_{n-d}$ of $\frakR$ we have
		\begin{align}\label{eq:dim_R/a}
			\dim(\frakR/\fraka+(\ell_1,\dots,\ell_{n-d}))=\dim(\frakR/\fraka)-(n-d)>0.
		\end{align} Geometrically, this means that the generic linear subspace $\cV$ defined as the common vanishing locus of the linear forms $\ell_1,\dots,\ell_{n-d}$ intersects $\cU^{\cl}$ at positive dimension, that is $\cV \cap \cU^{\cl} \supsetneq \{0\}$. 
		If $\cU=\cU^{\cl}$ we are done, so assume that $\cX:=\cU^{\cl} \backslash \cU$ is not empty. Since $\cU$ is open in $\cY$, we have that $\cX$ is closed in $\cU^{\cl}$. Suppose that $\dim(\cX)=\dim(\cU^{\cl})$. Let $\cX'$ be a maximal irreducible closed subset in $\cX$. Then necessarily $\cX'$ is an irreducible component of $\cU^{\cl}$. But $\cX' \cap \cU = \varnothing$, which contradicts the fact that $\cU^{\cl}$ is the smallest closed set that contain $\cU$. We conclude that $\dim(\cX)<\dim(\cU^{\cl})$. With $\frakb$ the vanishing ideal of $\cX$, we have $\dim(\frakR/\fraka)>\dim(\frakR/\frakb)$. Let us show that $\frakb$ is homogeneous. For any $z\in \cX\subset \cU^{\cl} \subset \cY$, we have $\lambda z\in \cU^{\cl}$ for any $\lambda\in\bbC$. Assume for the sake of contradiction that $\lambda' z\in \cU$ for some $\lambda' \in\bbC$. Note that $\lambda'$ can not be zero because $0 \in \cZ$. But $z\notin \cU$ implies $z\in \cZ$ and so $\lambda' z\in \cZ$, a contradiction. Taking again quotient by $n-d$ generic linear forms we have 
		\begin{align}\label{eq:dim_R/b}
			\dim(\frakR/\frakb+(\ell_1,\dots,\ell_{n-d}))=\max\{\dim(\frakR/\frakb)-(n-d), 0\}.
		\end{align}
		Combining \eqref{eq:dim_R/a}, \eqref{eq:dim_R/b} with $\dim(\frakR/\fraka)>\dim(\frakR/\frakb)$ we get
		\begin{align*}
			\dim(\frakR/\fraka+(\ell_1,\dots,\ell_{n-d}))>\dim(\frakR/\frakb+(\ell_1,\dots,\ell_{n-d})).
		\end{align*}
This implies $\dim(\cU^{\cl} \cap \cV)>\dim(\cX\cap\cV)$ for a generic $\cV\in\Gr_{\bbC}(d,n)$. Thus $\cV$ necessarily intersects $\cU$. \qed 	
	
\section*{Acknowledgments}
This work was funded by ShanghaiTech start-up grant 2017F0203-000-16. The authors thank Boshi Wang for useful discussions.

\end{document}